\PassOptionsToPackage{table}{xcolor}
\documentclass{article} 
\usepackage[final]{2026_conference}

\usepackage{microtype}
\usepackage{hyperref}
\usepackage{url}
\usepackage{booktabs}
\usepackage{wrapfig}
\usepackage{amsmath}
\usepackage{amssymb}
\usepackage{mathtools}
\usepackage{amsthm}
\usepackage{makecell}
\usepackage{graphicx}
\usepackage{subcaption}
\usepackage{multirow}
\usepackage{multicol}
\usepackage{pifont}
\usepackage[misc]{ifsym}
\usepackage{enumitem}
\usepackage{color}
\usepackage{xcolor}
\definecolor{mygreen}{RGB}{0,120,90}
\usepackage{algorithm}          
\usepackage{algpseudocode}      
\usepackage{enumitem}
\usepackage{multirow}
\usepackage{titlesec}
\usepackage{graphicx}
\usepackage{caption}

\usepackage[capitalize,noabbrev]{cleveref}

\definecolor{darkblue}{rgb}{0, 0, 0.5}
\hypersetup{colorlinks=true, citecolor=mid, linkcolor=primary, urlcolor=mid}

\renewcommand{\cite}[1]{\citep{#1}}

\usepackage{amsmath,amsfonts,bm}

\def\eqref#1{equation~\ref{#1}}

\def\1{\bm{1}}

\DeclareMathAlphabet{\mathsfit}{\encodingdefault}{\sfdefault}{m}{sl}
\SetMathAlphabet{\mathsfit}{bold}{\encodingdefault}{\sfdefault}{bx}{n}

\newcommand{\cmark}{\ding{51}}
\newcommand{\xmark}{\ding{55}}
\newtheorem{lemma}{Lemma}

\newtheorem{assumption}{Assumption}
\newtheorem{theorem}{Theorem}

\usepackage[export]{adjustbox} 

\newcommand{\B}[1]{\textbf{#1}}

\title{EMO: A Closed-Loop Test-Time Computation\\
Framework via Alternating Context\\
Optimization and Logits Steering}

\author{
\mbox{
 Wei Xia$^{1,2}$, Yehui Tang$^{2{~\textrm{\Letter}}}$ and Zhi-Hong Deng$^{1{~\textrm{\Letter}}}$}\\
$^1$ State Key Laboratory of General Artificial Intelligence, School of Intelligence Science and Technology, Peking University\\
$^2$ Samsung Research, Beijing, China\\
\texttt{xwisawesome}@stu.pku.edu.cn, \texttt{zhdeng}@pku.edu.cn\\
\texttt{yehui.tang}@samsung.com \\
$^\textrm{\Letter}$~Corresponding Authors
}

\begin{document}

\maketitle

\begin{abstract}
Test-time context optimization adapts frozen language models through textual feedback, but ordinary conditional generation provides no separate control over how strongly to reinforce the distributional changes induced by a context update.
We introduce \textbf{EMO}, a training-free framework inspired by expectation--maximization that alternates question-specific context estimation (E-step) and distribution refinement (M-step).
The E-step revises a system context using textual gradients from an external judge without access to gold answers.
The M-step contrasts next-token distributions under the revised and preceding contexts at a shared prefix to construct a confidence-constrained logit correction.
Feedback on the steered response drives the next context update, coupling adaptation across rounds with intervention within each generation.
Theoretically, we connect parameter and context optimization through a shared target-KL objective and additive logit representations, and derive the steering vector from context-induced logit displacements.
Under context-improvement and faithful-steering assumptions, we prove a one-round decrease in target KL divergence for sufficiently small steering strength.
Under greedy decoding across eight models and six benchmarks, EMO improves the zero-shot baseline in 39 of 40 aggregated settings, averaging $+9.10$ percentage points.
With the same judge, it adds $+1.53$ points on average over context-only optimization (TextGrad), with consistent MMLU-Pro gains across models.
At matched target-model generation budgets, accuracy advantages over best-of-$N$ are strongest on Llama, while oracle coverage is higher in most comparisons despite fewer candidates.
The coverage--selection gap motivates allocating test-time computation jointly to candidate generation and verification.
\end{abstract}

\section{Introduction}
\label{sec:introduction}

Parameter updates and context refinement provide complementary routes to adapting a large language model (LLM).
Parameter optimization, including fine-tuning, reinforcement learning, and preference alignment, incorporates task data or feedback into model weights~\citep{hu2021lora,shao2024deepseekmath,rafailov2023direct}.
For an individual test question, context optimization adapts a frozen model by using feedback to revise its textual instructions before generating a new response~\citep{yuksekgonul2025textgrad,zhang2026agentic,madaan2023self}.
We study how to translate such question-specific feedback into explicit control of the decoding process.

A revised context already changes the next-token distribution at every prefix.
Ordinary conditional generation, however, provides no separate control over the strength or token support of those changes.
Comparing the distributions before and after a context update reveals how the revision changes token preferences at the current prefix.
This motivates our central question: \emph{how can feedback-driven context updates supply a controllable intervention on the generation they guide?}

We propose \textbf{EMO}, which alternates context optimization and distribution refinement for each question.
The \emph{E-step} revises a system context using textual gradients~\citep{yuksekgonul2025textgrad} from an external judge without access to gold answers.
The \emph{M-step} compares the revised and preceding contexts at a shared generated prefix, constructs a steering vector from their distribution contrast, and applies a confidence-constrained correction to the revised context's logits.
The resulting response is evaluated to provide feedback for the next E-step.
Each context update thus determines a decoding intervention whose output informs the next update, closing the loop while model weights remain fixed.

A logit-space view explains the relationship between the two steps.
Parameter and context updates both admit additive logit representations, while EMO uses the context-induced displacement to construct its additional steering correction.
Context edits induce finite logit changes through a frozen model, whereas parameter updates also alter the model's response to subsequent contexts and prefixes.
This representation links adaptation to decoding, while our one-round analysis identifies sufficient conditions for improvement toward a target distribution.

Across eight models and six benchmarks, EMO improves over Base in 39 of 40 aggregated settings, averaging $+9.10$ percentage points, with benefits spanning model scale, generation, family, and thinking mode.
With the judge and candidate count held fixed, adding the M-step yields $+1.53$ points on average over context-only optimization and improves MMLU-Pro for all eight models.
At matched target-model generation budgets, EMO also achieves higher oracle coverage than independent sampling in seven of nine settings despite producing fewer candidates, indicating benefits in candidate generation alongside improvements in selected accuracy.

\noindent Our contributions are summarized as follows:
\begin{itemize}[leftmargin=*, nosep]
    \item \textbf{EMO: coupled context optimization and steering.} A training-free framework that turns successive context updates into prefix-dependent decoding interventions and uses the steered responses to drive further context adaptation.

    \item \textbf{Logit-space interpretation and conditional analysis.} We relate the context-induced logit displacement to the implemented steering vector (Eq.~\ref{eq:context_displacement}) and establish one-round target-KL improvement under context-improvement and faithful-steering assumptions (Theorem~\ref{thm:monotonicity}).

    \item \textbf{Empirical evidence across four model axes.} Broad improvements over Base and additional gains over context-only optimization, supported by component ablations, decoding comparisons, and generation-budget analyses that distinguish candidate coverage from answer selection.
\end{itemize}

\section{Related Work}

\textbf{Parameter optimization.}
Parameter adaptation includes supervised fine-tuning and its parameter-efficient variants~\citep{hu2021lora}, reinforcement learning~\citep{schulman2017proximal,shao2024deepseekmath,guo2025deepseek,Yu2025DAPOAO}, preference optimization~\citep{rafailov2023direct}, and distillation~\citep{snell2022learning,agarwal2024policy}.
These approaches incorporate training signals into model weights.
EMO operates on a frozen checkpoint and uses feedback on the current question to adapt its context and decoding, making it complementary to parameter adaptation.

\textbf{Context optimization.}
A second line adapts textual inputs while keeping model weights frozen, including textual gradients and prompt optimization~\citep{yuksekgonul2025textgrad,yang2024large}, self-correction~\citep{madaan2023self,shinn2023reflexion,gou2024critic}, and agentic context engineering that evolves skills and memories~\citep{zhang2026agentic,ye2026meta}.
These approaches use textual revisions to influence subsequent generation.
EMO's E-step builds on this idea; its M-step additionally uses the distributional change between successive contexts to control decoding at each prefix.
The steered response then becomes the basis for further textual feedback.

\textbf{Inference-time intervention.}
A third family controls generation at inference time by modifying internal activations~\citep{Turner2023SteeringLM} or contrasting next-token distributions during decoding~\citep{li2023contrastive}.
Closely related task-aware approaches include CoS~\citep{he2025context}, SDA~\citep{xia2026sda}, SVDecode~\citep{hu2026distribution}, and JitRL~\citep{li2026just}.
EMO connects decoding-time intervention to per-question context optimization: each response informs the next context update, whose distribution contrast supplies the steering signal at every subsequent prefix.
The logit-space view in Section~\ref{sec:unified_view} formalizes how these context updates determine prefix-dependent decoding interventions.
Test-time compute can also be spent on parallel sampling, including self-consistency~\citep{wang2022self} and best-of-$N$ with a verifier or judge~\citep{Cobbe2021TrainingVT,brown2024large}, or sequential refinement~\citep{snell2024scaling}.
Our experiments compare EMO with context-only refinement and, separately, with SC/BoN at matched target-model generation budgets (Section~\ref{sec:compute}).

\section{Method}
\label{sec:method}

EMO couples context refinement and decoding interventions through feedback on generated responses.
Context updates induce question-specific changes in token preferences; steering converts these changes into controlled logit corrections at each decoding prefix.
We first relate the two updates in logit space, then describe the algorithm and the conditions for target alignment.

\subsection{A Unified View: Optimization as Logits Steering}
\label{sec:unified_view}

Let $\xi=(\boldsymbol{\theta},C)$ denote model parameters and system context jointly, and let $\mathcal V$ be the token vocabulary. The response distribution is
$P_\xi(y\mid q)=\prod_t P_\xi(y_t\mid q,y_{<t})$, where $q$ is a question and $y_t,y_{<t}$ denote token $t$ and its prefix in response $y$.
Numerical vectors use bold lowercase and matrices bold uppercase; scalars are italic.
Distributions retain uppercase symbols, interpreted as probability columns in matrix operations.

\paragraph{A shared alignment objective.}
\label{sec:unified_nll}
\label{sec:unified_kl}
SFT minimizes reference-response negative log-likelihood (NLL) over $\boldsymbol{\theta}$ with $C$ fixed.
TextGrad~\citep{yuksekgonul2025textgrad} instead refines text from loss feedback with $\boldsymbol{\theta}$ frozen; when instantiated with reference NLL, it targets the same objective through $C$.

\begin{lemma}[Unified Objective: NLL--KL Equivalence]
\label{lem:unified_objective}
With one reference response per question,
\begin{equation}
 \mathcal L(\xi)
 =\mathbb E_{(q,y^\star)\sim\mathcal D}[-\log P_\xi(y^\star\mid q)]
 =\mathbb E_{q\sim\mathcal D}\mathrm{KL}\!\left(
 \hat P_{\mathrm{target}}(\cdot\mid q)\,\|\,P_\xi(\cdot\mid q)\right),
 \label{eq:unified_objective}
\end{equation}
where $\mathcal D$ is the empirical distribution of question--reference pairs, $y^\star=y^\star(q)$ is the reference, $\mathcal L$ is the mean NLL, and $\hat P_{\mathrm{target}}(\cdot\mid q)=\delta_{y^\star(q)}$ is a point mass at that reference.
\end{lemma}

\begin{proof}
The target assigns probability one to $y^\star(q)$, so its KL is $\log[1/P_\xi(y^\star\mid q)]$.
Averaging over questions gives Eq.~\ref{eq:unified_objective}.
For a general fixed target, cross-entropy differs from target KL only by the target entropy (Appendix~\ref{app:objective09}).
\end{proof}

This identity provides a common objective, not identical feasible distributions or update directions for parameter and context optimization.
At test time, judge feedback provides a surrogate signal for this unavailable reference-based objective; textual gradients encode natural-language criticism rather than derivatives of strings.

\paragraph{A shared execution form.}
\label{sec:unified_steering}
Consider a shared history $h=(q,y_{<t})$, with finite logits $\mathbf{z}$, model conditional $P=\mathrm{Softmax}(\mathbf{z})$, and target conditional $Q=\hat P_{\mathrm{target}}^t(\cdot\mid h)$ at a target-supported prefix.

\begin{lemma}[Unified Steering Form]
\label{lem:unified_steering}
Any parameter or context update $\xi\to\xi^+$ induces an exact additive correction $\Delta \mathbf{z}=\mathbf{z}_{\xi^+}-\mathbf{z}_\xi$, yielding $\mathrm{Softmax}(\mathbf{z}_\xi+\Delta \mathbf{z})$ at each history.
In unrestricted logit space, gradient descent on the conditional target KL gives
\begin{equation}
 \hat{\mathbf{z}}=\mathbf{z}+\eta\boldsymbol{\delta}^*,\qquad
 \boldsymbol{\delta}^*=\mathbf J_P\frac{Q}{P}=Q-P,
 \qquad \mathbf J_P=\mathrm{diag}(P)-PP^\top,
\label{eq:steering_update}
\end{equation}
where $\hat{\mathbf{z}}$ contains updated logits, $\eta>0$ is an analysis step size, $\boldsymbol{\delta}^*$ is the ideal descent direction, and $\mathbf J_P$ is the softmax Jacobian. Division is componentwise; $\mathrm{diag}(P)$ is the diagonal probability matrix.
If $Q\ne P$, every sufficiently small $\eta>0$ decreases the conditional KL.
\end{lemma}

\begin{proof}
The additive realization follows by substituting the logit difference.
Since $\nabla_P\mathrm{KL}(Q\|P)=-Q/P$ and $\mathbf J_P$ is symmetric, the chain rule gives $-\nabla_{\mathbf z}\mathrm{KL}(Q\|P)=\mathbf J_P(Q/P)=Q-P$.
The directional derivative along $\boldsymbol{\delta}^*$ is therefore $-\|Q-P\|_2^2<0$ when $Q\ne P$; differentiability gives the stated local decrease.
\end{proof}

Parameter gradients propagate through the model Jacobian, whereas textual rewrites induce finite logit differences; neither generally follows the direct-logit direction $Q-P$ (Appendix~\ref{app:steering09}).
Since $Q$ is unavailable, EMO constructs a computable direction from the context-induced distribution contrast.

\subsection{EMO: Alternating Context Optimization and Logits Steering}
\label{sec:emopt}

EMO alternates \textbf{context estimation (E-step)} and \textbf{distribution refinement (M-step)} per question.
The EM analogy is structural: textual feedback is not a posterior expectation, and steering does not maximize a specified common likelihood.

\subsubsection{E-Step: Context Optimization via Textual Gradients}
\label{sec:estep}

For each question, initialize $C_0$ to the shared system context and generate $R_0=\mathrm{Generate}(\boldsymbol{\theta},C_0,q)$ using ordinary target-model decoding.
Judge feedback guides the context update:
\begin{equation}
 C_{k+1}=\mathrm{Opt}\!\left(C_k,g_C^{(k)}\right),\qquad
 g_C^{(k)}=\mathrm{Backward}(C_k,q,R_k,F_k),
\label{eq:estep}
\end{equation}
where $k$ indexes rounds, $R_k$ is the current response, and $F_k$ is feedback from a judge without gold-answer access.
$\mathrm{Backward}$ translates the judge's criticism of $R_k$ into suggested changes to $C_k$, producing the textual gradient $g_C^{(k)}$; $\mathrm{Opt}$ applies this guidance to construct $C_{k+1}$.
Subsequent rounds use the preceding M-step's steered response as $R_k$.
Rewrites update the domain role, reasoning strategies, pitfalls, and output requirements while keeping the question fixed.
Feedback and rewrite prompts forbid final-answer disclosure; a pattern filter rejects suspected answer-bearing rewrites by setting $C_{k+1}=C_k$.
Target-KL improvement is an analysis assumption (Section~\ref{sec:theory}), not guaranteed by the textual update.

\subsubsection{M-Step: Distribution Contrast as Logits Steering}
\label{sec:mstep}

At each decoding step, evaluate $P_{\mathrm{opt}}=P_{\boldsymbol{\theta}}(\cdot\mid C_{k+1},q,y_{<t})$ and $P_{\mathrm{base}}=P_{\boldsymbol{\theta}}(\cdot\mid C_k,q,y_{<t})$ on the same generated prefix.
Here ``base'' denotes the preceding context $C_k$, not the initial context $C_0$; round, question, and prefix indices are suppressed.
Both streams consume the newly generated prefix rather than their own independently sampled responses, isolating the effect of the context revision.
Following SVDecode~\citep{hu2026distribution}, the contrast yields:
\begin{equation}
 \boldsymbol{\delta}_t=\mathbf J_{P_{\mathrm{opt}}}\log\frac{P_{\mathrm{opt}}}{P_{\mathrm{base}}}
 =\nabla_{\mathbf{z}_t^{\mathrm{opt}}}\mathrm{KL}(P_{\mathrm{opt}}\|P_{\mathrm{base}}),
 \label{eq:mstep_delta}
\end{equation}
where $\boldsymbol{\delta}_t$ is the raw steering vector, $\mathbf{z}_t^{\mathrm{opt}}$ contains logits under $C_{k+1}$, and $\mathbf J_{P_{\mathrm{opt}}}$ is the softmax Jacobian; differentiation holds $P_{\mathrm{base}}$ fixed.
Logarithms and ratios act componentwise.
For nonzero $\boldsymbol{\delta}_t$, an unmasked positive step locally increases divergence from the preceding-context distribution; this alone does not ensure target alignment.

\paragraph{From context updates to steering.}
At a shared prefix, context updates give:
\begin{equation}
 \Delta \mathbf{z}_t=\mathbf{z}_t^{\mathrm{opt}}-\mathbf{z}_t^{\mathrm{base}},\qquad
 \boldsymbol{\delta}_t=\mathbf J_{P_{\mathrm{opt}}}\Delta \mathbf{z}_t,
\label{eq:context_displacement}
\end{equation}
where $\mathbf{z}_t^{\mathrm{base}}$ contains logits under $C_k$ and $\Delta\mathbf{z}_t$ is the context-induced displacement.
This identity follows because $\log(P_{\mathrm{opt}}/P_{\mathrm{base}})$ differs from $\Delta \mathbf{z}_t$ only by a constant across tokens, which the softmax Jacobian annihilates.
The Jacobian centers and probability-weights the displacement from Lemma~\ref{lem:unified_steering}.
Writing $r_t(v)=\log P_{\mathrm{opt}}(v)-\log P_{\mathrm{base}}(v)$ gives the implementation directly:
\begin{equation}
 \delta_t(v)=P_{\mathrm{opt}}(v)\left[r_t(v)-\sum_{w\in\mathcal V}P_{\mathrm{opt}}(w)r_t(w)\right],
 \label{eq:centered_contrast_main}
\end{equation}
where $v,w\in\mathcal V$ index tokens and $r_t$ is the context log-ratio.
A positive correction requires an above-average log-ratio, not merely $P_{\mathrm{opt}}(v)>P_{\mathrm{base}}(v)$.
Using log-softmax values, this expression costs $O(|\mathcal V|)$ per prefix without a dense Jacobian or parameter backpropagation (Appendix~\ref{app:contrast09}).

\paragraph{Confidence-constrained decoding.}
To restrict the contrast signal to tokens supported by the optimized context, set
$\mathbb I_t(v)=\mathbf 1\{P_{\mathrm{opt}}(v)\geq\alpha\max_{v'}P_{\mathrm{opt}}(v')\}$ and apply
\begin{equation}
 \hat\delta_t(v)=\mathbb I_t(v)\delta_t(v)
 +(1-\mathbb I_t(v))\lambda,
 \qquad \alpha\in(0,1],\quad\lambda=-100,
\label{eq:confidence_mask}
\end{equation}
where $v,v'$ are tokens, $\mathbb I_t$ indicates whether a token passes the relative threshold $\alpha$, $\lambda$ penalizes exclusion, and $\delta_t(v),\hat\delta_t(v)$ are entries of $\boldsymbol{\delta}_t,\hat{\boldsymbol{\delta}}_t$.
Increasing $\alpha$ can only shrink the retained set, which always includes a maximizer; excluded tokens receive finite penalties.
The constrained vector then corrects the logits used for greedy decoding or sampling:
\begin{equation}
 \tilde{\mathbf{z}}_t=\mathbf{z}_t^{\mathrm{opt}}+\mu\hat{\boldsymbol{\delta}}_t,\qquad
 y_t=\mathrm{Decode}\!\left(\mathrm{Softmax}(\tilde{\mathbf{z}}_t)\right),
 \label{eq:mstep_logits}
\end{equation}
where $\tilde{\mathbf{z}}_t$ contains corrected logits, $\mu\geq0$ scales the constrained vector, and $\mathrm{Decode}$ selects $y_t$ greedily or by sampling.
The defaults are $(\mu,\alpha)=(2.4,0.1)$; excluded tokens receive the additive correction $\mu\lambda$.
At $\mu=0$, the method reduces to context-only optimization.
Both contrast and masking are recomputed at each prefix, before sampling temperature or truncation is applied.

\paragraph{Closed-loop execution.}
\label{sec:algorithm}
Algorithm~\ref{alg:emopt} performs $K$ optimization rounds, generating at most $T$ tokens per response and stopping at the end-of-sequence (EOS) token.
Feedback always concerns the most recent response; the highest-scoring candidate is used only for final selection.
It returns the earliest response maximizing the stored score $a_i=\mathcal J(q,R_i)$, where $\mathcal J$ is the judge and $i\in\{0,\ldots,K\}$ indexes candidates.
Two context-conditioned streams maintain separate key--value caches while sharing the generated prefix.
Thus $K=4$ yields five candidates and nine target-model generation equivalents; this excludes judge/rewrite costs and response-length differences and does not imply equal latency.

\subsection{Theoretical Analysis}
\label{sec:theory}

Fix a question, the contexts $C_k,C_{k+1}$, a finite vocabulary, and a maximum horizon $T$.
Let $Q$ now denote a fixed target \emph{sequence} distribution, $Q_t(\cdot\mid h)$ its conditionals, and $\widetilde P_\mu$ the sequence distribution induced by Eq.~\ref{eq:mstep_logits} before temperature scaling or truncation.
Assume finite logits at nonterminal histories; after EOS, use shared deterministic padding with zero steering.

\begin{assumption}[TextGrad Oracle]
\label{ass:oracle}
The context update improves the target objective: $\mathrm{KL}(Q\|P_{\boldsymbol{\theta},C_{k+1}})<\mathrm{KL}(Q\|P_{\boldsymbol{\theta},C_k})$.
\end{assumption}

\begin{assumption}[Faithful Constrained Steering]
\label{ass:steering}
For fixed $\alpha$, the actual masked direction has positive target-weighted alignment:
\begin{equation}
 \Gamma=\sum_{t=1}^{T}\mathbb E_{Y_{<t}\sim Q}
 \left\langle Q_t(\cdot\mid h)-P_{\mathrm{opt}}(\cdot\mid h),\hat{\boldsymbol{\delta}}_t(h)\right\rangle>0.
 \label{eq:alignment_main}
\end{equation}
Here histories follow $Q$, and $\Gamma$ aggregates alignment across positions; individual histories need not all improve.
\end{assumption}

\begin{theorem}[Conditional One-Round Improvement]
\label{thm:monotonicity}
Under the stated finite-horizon setting and Assumptions~\ref{ass:oracle}--\ref{ass:steering}, every sufficiently small $\mu>0$ gives $\mathrm{KL}(Q\|\widetilde P_\mu)<\mathrm{KL}(Q\|P_{\boldsymbol{\theta},C_{k+1}})<\mathrm{KL}(Q\|P_{\boldsymbol{\theta},C_k})$.
\end{theorem}

\begin{proof}
At a fixed history, $\hat{\boldsymbol{\delta}}_t(h)$ does not depend on $\mu$.
The KL chain rule and the conditional gradient in Lemma~\ref{lem:unified_steering} give
\begin{equation}
 \left.\frac{d}{d\mu}\mathrm{KL}(Q\|\widetilde P_\mu)\right|_{\mu=0}
 =\sum_{t=1}^{T}\mathbb E_{Y_{<t}\sim Q}
 \left\langle P_{\mathrm{opt}}-Q_t,\hat{\boldsymbol{\delta}}_t\right\rangle=-\Gamma<0.
 \label{eq:sequence_derivative_main}
\end{equation}
Finite vocabulary and horizon justify differentiation under the sums.
Since $\widetilde P_0=P_{\boldsymbol{\theta},C_{k+1}}$, differentiability gives the first inequality; Assumption~\ref{ass:oracle} gives the second.
\end{proof}

Appendix~\ref{app:theory09} gives an explicit sufficient step-size interval and shows why context improvement alone does not imply faithful steering.
The guarantee concerns full softmax distributions and unsteered context baselines; it certifies neither the default $\mu=2.4$, monotonic accuracy, nor improvement between steered rounds.
Judge selection preserves the best stored score, not necessarily correctness (Appendix~\ref{app:selection09}).

\par\addvspace{\intextsep}
\noindent
\begin{minipage}{\linewidth}
\begin{minipage}[t]{0.49\textwidth}
\vspace{0pt}
\setlength{\intextsep}{0pt}
\begin{algorithm}[H]
\caption{EMO}
\label{alg:emopt}
\small
\begin{algorithmic}[1]
\Require Model $\boldsymbol{\theta}$, question $q$, context $C_0$, judge $\mathcal J$, rounds $K$, token limit $T$, $\mu$, $\alpha$
\State $R_0\gets\mathrm{Generate}(\boldsymbol{\theta},C_0,q)$
\State $a_0\gets\mathcal J(q,R_0)$
\For{$k=0,\ldots,K-1$}
 \State $F_k\gets\mathrm{Feedback}_{\mathcal J}(q,R_k)$
 \State $g_C^{(k)}\gets\mathrm{Backward}(C_k,q,R_k,F_k)$
 \State $C_{k+1}\gets\mathrm{Opt}(C_k,g_C^{(k)})$
 \State Set $C_{k+1}\gets C_k$ if the rewrite is rejected
 \State Initialize empty prefix $y_{<1}$
 \For{$t=1,\ldots,T$}
  \State Evaluate $P_{\mathrm{opt}},P_{\mathrm{base}}$ at $y_{<t}$
  \State Compute $\boldsymbol{\delta}_t$ by Eq.~(\ref{eq:mstep_delta})
  \State Apply $\alpha$-mask by Eq.~(\ref{eq:confidence_mask})
  \State $\tilde{\mathbf{z}}_t\gets \mathbf{z}_t^{\mathrm{opt}}+\mu\hat{\boldsymbol{\delta}}_t$
  \State $y_t\gets\mathrm{Decode}(\mathrm{Softmax}(\tilde{\mathbf{z}}_t))$
  \State Append $y_t$; stop on EOS
 \EndFor
 \State $R_{k+1}\gets$ generated response
 \State $a_{k+1}\gets\mathcal J(q,R_{k+1})$
\EndFor
\State \Return earliest highest-scoring $R_i$
\end{algorithmic}
\end{algorithm}
\end{minipage}\hfill
\begin{minipage}[t]{0.49\textwidth}
\vspace{0pt}
\centering
\includegraphics[width=\linewidth]{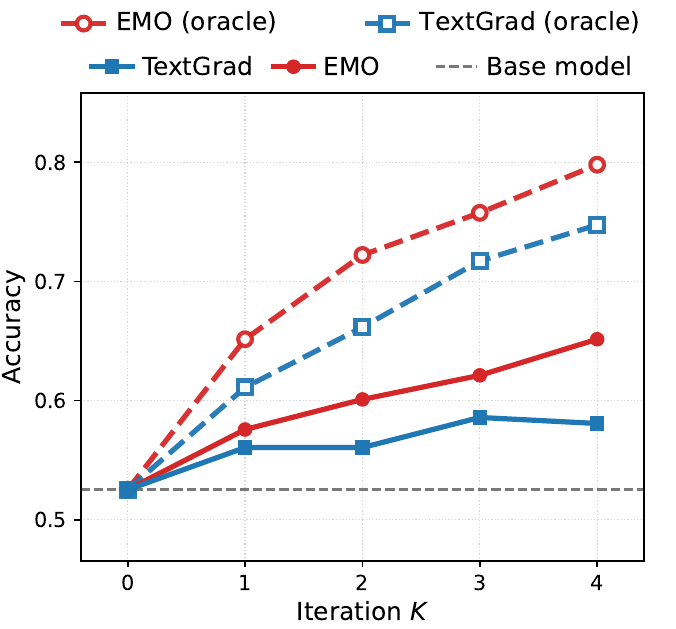}
\captionsetup{font=small,hypcap=false}
\captionof{figure}{Accuracy and oracle coverage versus iteration on Qwen3-14B/GPQA-Diamond. Solid markers show judge-selected accuracy; hollow markers show oracle coverage over $K+1$ candidates, including the initial response. These trajectories are empirical.}
\label{fig:iter_curve}
\end{minipage}
\end{minipage}
\par\addvspace{\intextsep}


\section{Experiments}
\label{sec:experiments}

\subsection{Experimental Setup}
\label{sec:setup}

\paragraph{Models and benchmarks.}
We evaluate eight models spanning 3B--32B parameters and two families: Qwen3-8B/14B/32B, Qwen2.5-7B, Qwen3.5-9B, DeepSeek-R1-0528-Qwen3-8B, Llama-3.1-8B-Instruct, and Llama-3.2-3B-Instruct, covering both thinking and non-thinking models (Appendix~\ref{app:models}).
We use six benchmarks covering multiple-choice and free-form formats: AIME-2025, AIME-2026, GPQA-Diamond, GSM8K, MATH-500, and MMLU-Pro (560 questions: 40 per category, sampling seed 42).
The main table averages the two AIME editions into one of five equally weighted task groups (Appendix~\ref{app:datasets}).

\paragraph{Judge and scoring.}
The external judge, Qwen3.5-122B-A10B (FP8), assigns scores from 0 to 10 without access to gold answers; feedback prompts forbid final-answer disclosure.
Judge prompting and decoding settings are detailed in Appendix~\ref{app:settings}.
Reported accuracy uses task-specific answer extraction and gold-based correctness rules, independently of judge scores.
Under \emph{judge-select}, the highest-scoring response among the initial generation and $K$ optimization rounds is selected, with ties resolved in favor of the earliest response.
We separately report \emph{oracle coverage}: the fraction of questions with at least one correct response in the candidate pool.

\paragraph{Baselines and protocol.}
We compare \textbf{Base} (zero-shot decoding) and \textbf{TextGrad} (our E-step without the M-step), adapting TextGrad from task-level to per-question optimization to match our protocol.
TextGrad and EMO share the initial context, baseline response, judge, and $K=4$ optimization rounds, with judge-select over five candidates.
EMO uses fixed defaults $\mu=2.4$ and $\alpha=0.1$ (Appendix~\ref{app:hyperparam}).
Main results use greedy target decoding; budget comparisons use temperature 0.6, top-$p$ 0.95, and seeds $\{0,1,2\}$ against self-consistency (SC) and best-of-$N$ (BoN) (Section~\ref{sec:compute}). Paired significance tests compare EMO at $N=9$ with Base and BoN@10 (Appendix~\ref{app:significance}).
Full decoding settings and token limits appear in Appendix~\ref{app:settings}.


\begin{table}[t]
\centering
\caption{\textbf{Main results.} Accuracy (\%) across 8 models and 6 benchmarks (AIME-2025/2026 merged into one task group), greedy decoding. Each group reports the zero-shot baseline (Base), context-only optimization (TextGrad), and EMO; the best per group is bolded.}
\label{tab:main}
\setlength{\tabcolsep}{1.8pt}
\small
\begin{tabular}{l ccc ccc ccc ccc ccc}
\toprule
& \multicolumn{3}{c}{AIME (\%)} & \multicolumn{3}{c}{GPQA-D (\%)} & \multicolumn{3}{c}{GSM8K (\%)} & \multicolumn{3}{c}{MATH-500 (\%)} & \multicolumn{3}{c}{MMLU-Pro (\%)} \\
\cmidrule(lr){2-4}\cmidrule(lr){5-7}\cmidrule(lr){8-10}\cmidrule(lr){11-13}\cmidrule(lr){14-16}
Model & Base & TG & EMO & Base & TG & EMO & Base & TG & EMO & Base & TG & EMO & Base & TG & EMO \\
\midrule
Llama-3.2-3B   & 0.0 & \B{6.7} & 3.3 & 26.8 & \B{40.4} & 34.3 & 71.9 & 86.5 & \B{87.1} & 39.8 & \B{61.0} & 60.4 & 34.8 & 45.7 & \B{53.0} \\
Qwen2.5-7B     & 5.0 & 11.7 & \B{16.7} & 32.8 & 42.9 & \B{45.5} & 87.8 & 90.7 & \B{91.7} & 70.0 & 78.0 & \B{79.2} & 54.8 & 61.6 & \B{65.5} \\
Llama-3.1-8B   & 0.0 & \B{3.3} & \B{3.3} & 29.3 & \B{43.9} & 38.9 & 82.1 & \B{89.6} & 89.4 & 43.2 & 65.4 & \B{65.6} & 45.4 & 53.9 & \B{63.9} \\
Qwen3-8B       & 20.0 & 43.3 & \B{46.7} & 49.0 & 54.5 & \B{56.6} & 91.4 & \B{92.7} & 92.3 & 75.0 & 82.0 & \B{83.0} & 63.6 & 64.8 & \B{67.5} \\
DR1-Qwen3-8B   & 8.3 & \B{18.3} & 16.7 & 46.0 & \B{48.0} & \B{48.0} & 82.2 & \B{88.1} & 87.9 & 56.2 & \B{66.6} & 64.0 & 62.0 & 67.0 & \B{68.2} \\
Qwen3.5-9B     & 25.0 & 26.7 & \B{45.0} & 39.4 & 41.9 & \B{45.5} & 89.0 & \B{90.6} & \B{90.6} & 48.2 & 50.0 & \B{50.4} & 62.9 & 63.6 & \B{65.0} \\
Qwen3-14B      & 20.0 & \B{33.3} & \B{33.3} & 52.5 & 58.1 & \B{65.2} & 91.7 & \B{92.0} & 91.5 & 76.2 & 83.2 & \B{83.8} & 71.4 & 73.6 & \B{75.2} \\
Qwen3-32B      & 18.3 & 43.3 & \B{48.3} & 59.6 & \B{64.6} & 62.1 & 91.4 & 91.9 & \B{92.0} & 78.6 & 83.8 & \B{86.0} & 71.4 & 72.3 & \B{74.3} \\
\midrule
Average        & 12.1 & 23.3 & \B{26.7} & 41.9 & 49.3 & \B{49.5} & 85.9 & 90.3 & \B{90.3} & 60.9 & 71.3 & \B{71.6} & 58.3 & 62.8 & \B{66.6} \\
\bottomrule
\end{tabular}
\end{table}


\subsection{Main Results}
\label{sec:main_results}

\paragraph{Broad gains across the model grid.}
Table~\ref{tab:main} separates overall gains over Base from incremental gains over context-only optimization.
\textbf{EMO improves over Base in 39 of 40 cells} (47 of 48 before merging the two AIME editions), averaging $+9.10$ percentage points across the 40 model--task-group pairs.
All eight models improve on average; EMO also exceeds TextGrad in 25 cells and ties in four, with a mean increment of $+1.53$ points.
Using Qwen3-8B as a common reference, we organize the analysis along four axes: model scale, generation, family, and thinking mode.

\paragraph{Scale: complementing model scaling.}
We first test whether adaptation benefits persist as target models scale.
Qwen3-14B with EMO exceeds Qwen3-32B Base on all five task groups, averaging $69.80$ versus $63.86$.
Likewise, optimized Llama-3.2-3B exceeds Llama-3.1-8B Base on every task group.
Larger models also retain substantial gains: Qwen3-32B improves by $8.68$ points on average and $30.0$ on AIME.
Accuracy headroom alone is insufficient to predict improvement: Qwen3 gains little on near-saturated GSM8K, yet Llama gains little on AIME despite near-zero baselines.
Mean gains across Qwen3-8B/14B/32B are non-monotonic ($9.42/7.44/8.68$): model size alone does not determine adaptation benefit.
Adaptation can close size-related gaps, but its benefit depends on whether feedback-guided generation can reach correct solutions.

\paragraph{Generation: sustained gains as models evolve.}
Across Qwen2.5-7B, Qwen3-8B, and Qwen3.5-9B, EMO yields mean gains of $9.64$, $9.42$, and $6.40$ points over Base.
The incremental gains over TextGrad are $2.74$, $1.76$, and $4.74$ points, respectively: the newest generation has the smallest overall gain but the largest M-step increment.
On Qwen3.5-9B, context-only optimization gains just $1.66$ points on average, whereas the full method gains $6.40$.
AIME makes this contrast concrete: Base, TextGrad, and EMO score $25.0$, $26.7$, and $45.0$.
Thus, limited gains from context rewrites need not imply limited scope for adaptation: steering can further exploit the distributional changes induced by those rewrites.

\paragraph{Family: narrowing gaps at matched model size.}
Qwen3-8B and Llama-3.1-8B allow a cross-family comparison at matched parameter scale.
On MMLU-Pro, their accuracy gap shrinks from $18.2$ points at Base to $3.6$ after applying EMO to both models.
The corresponding gaps shrink from $31.8$ to $17.4$ on MATH-500 and from $9.3$ to $2.9$ on GSM8K.
Steering contributes directly to the MMLU-Pro result: its increment over TextGrad is $10.0$ points for Llama versus $2.7$ for Qwen.
On Llama-3.1-8B/MATH-500, however, TextGrad already reaches $65.4$ versus EMO's $65.6$.
The narrowed family gaps therefore reflect task-dependent contributions from context refinement and steering.

\paragraph{Thinking mode: gains with and without explicit reasoning modes.}
The non-thinking Qwen2.5-7B, Llama-3.1-8B, and Llama-3.2-3B gain $9.64$, $12.22$, and $12.96$ points on average.
Qwen2.5-7B also exceeds TextGrad on all five task groups, showing that M-step gains do not require a native thinking mode.
Models with thinking modes also improve: the Qwen3 series gains $7.44$--$9.42$ points, and the reasoning-distilled DR1-Qwen3-8B gains $6.02$.
DR1 improves over Base on all five task groups, showing that reasoning distillation leaves scope for question-specific adaptation.
Yet its M-step increment is $-0.64$ points on average, with an additional gain only on MMLU-Pro ($68.2$ vs.\ $67.0$).
DR1 and Qwen3.5-9B therefore have similar overall gains ($6.02$ vs.\ $6.40$) but contrasting M-step contributions ($-0.64$ vs.\ $+4.74$).
One possible explanation is that distillation reshapes context-conditioned distributions and their sensitivity to logit interventions, altering the benefit of a fixed steering rule.
Establishing this mechanism requires matched pre-/post-distillation controls.

\paragraph{Where the M-step is most effective.}
Across these axes, \textbf{MMLU-Pro provides the most consistent steering gains}: all eight models exceed TextGrad, by $3.76$ points on average.
MMLU-Pro oracle coverage also improves over TextGrad for all eight models (mean $+9.6$), whereas mean oracle increments on MATH-500 and GSM8K are at most $1.4$ points (Appendix~\ref{app:iter_curves}).
At equal candidate counts, these results associate steering gains with broader coverage; mixed GPQA-D results show that multiple-choice format alone does not explain the effect.
Figure~\ref{fig:iter_curve} shows how the advantage develops across rounds on Qwen3-14B/GPQA-D.
From a shared $52.5$ baseline, EMO continues improving to $65.2$, while TextGrad levels off near $56$--$59$ and ends at $58.1$.
EMO's oracle curve also remains above TextGrad's after the first update, reaching $79.8$ at $K=4$.
Because oracle coverage does not depend on final answer selection, the curves indicate that improved candidate generation contributes to the $7.1$-point accuracy advantage.
We next analyze components, generation budgets, and selection (Section~\ref{sec:iteration_selection}).


\subsection{Component Ablations and Robustness}
\label{sec:ablations}

\begin{table}[t]
\noindent
\begin{minipage}[t]{0.48\textwidth}
\centering
\caption{\textbf{Component ablation} on Qwen3-8B (accuracy \%, greedy). ``w/o M-step'' reduces to context-only optimization (TextGrad); ``w/o Judge'' replaces the external judge with self-judging; ``w/o Both'' combines these changes.}
\label{tab:component}
\setlength{\tabcolsep}{3pt}
\small
\begin{tabular}{l ccccc}
\toprule
Dataset & Base & EMO & \shortstack{w/o\\M-Step} & \shortstack{w/o\\Judge} & \shortstack{w/o\\Both} \\
\midrule
AIME     & 20.0 & \B{46.7} & 43.3 & 31.7 & 25.0 \\
GPQA-D   & 49.0 & \B{56.6} & 54.5 & 51.0 & 54.5 \\
GSM8K    & 91.4 & 92.3 & \B{92.7} & 91.8 & 92.4 \\
MATH-500 & 75.0 & \B{83.0} & 82.0 & 78.6 & 78.0 \\
MMLU-Pro & 63.6 & \B{67.5} & 64.8 & 65.2 & 63.7 \\
\midrule
$\overline{\Delta}$ & -- & \B{+9.4} & +7.7 & +3.9 & +2.9 \\
\bottomrule
\end{tabular}
\end{minipage}%
\hfill
\begin{minipage}[t]{0.48\textwidth}
\centering
\caption{\textbf{Decoding-strategy ablation.} EMO gains over Base
(percentage points; single-seed results, seed~0) under greedy,
top-$p$ ($p=0.95$, $k=0$), and top-$k$ ($k=20$).
EMO and Base share decoding settings, with sampling temperature $0.6$.}
\label{tab:decoding}
\setlength{\tabcolsep}{3pt}
\small
\begin{tabular}{ll ccc}
\toprule
Dataset & Model & Greedy & top-$p$ & top-$k$ \\
\midrule
\multirow{2}{*}{\begin{tabular}[t]{@{}l@{}}GPQA-\\D\end{tabular}}
 & Qwen3-8B     & +7.6 & +4.5 & +6.6 \\
 & Llama-3.1-8B & +9.6 & +10.1 & +12.1 \\
\midrule
\multirow{2}{*}{\begin{tabular}[t]{@{}l@{}}MATH-\\500\end{tabular}}
 & Qwen3-8B     & +8.0 & +5.6 & +4.4 \\
 & Llama-3.1-8B & +22.4 & +20.6 & +21.6 \\
\midrule
\multirow{2}{*}{\begin{tabular}[t]{@{}l@{}}MMLU-\\Pro\end{tabular}}
 & Qwen3-8B     & +3.9 & +6.1 & +5.7 \\
 & Llama-3.1-8B & +18.6 & +21.8 & +25.2 \\
\bottomrule
\end{tabular}
\end{minipage}
\end{table}

\paragraph{External judging supplies the larger average contribution.}
Table~\ref{tab:component} varies M-step inclusion and judge choice jointly to distinguish their contributions on Qwen3-8B.
With the M-step retained, replacing self-judging with the external judge raises average accuracy by $5.56$ points; without the M-step, the increase is $4.74$ points.
With steering, the external judge's largest advantages are on AIME ($46.7$ vs.\ $31.7$) and GPQA-D ($56.6$ vs.\ $51.0$).
Since judge choice affects both feedback and response selection, these differences measure their combined effect.

\paragraph{Steering adds to context optimization under both judge choices.}
Holding judge choice fixed, adding the M-step improves average accuracy by $1.76$ points with the external judge and $0.94$ with self-judging.
These contributions interact: on Qwen3-8B/GPQA-D, steering gains $2.1$ points with the external judge but loses $3.5$ with self-judging.
On Qwen3-32B/GPQA-D, however, steering loses $2.5$ under either judge (Appendix~\ref{app:component_32b}).
These results show that steering effectiveness depends on both feedback quality and the target model, rather than on judge choice alone.

\paragraph{Robustness across decoding settings.}
Table~\ref{tab:decoding} tests whether gains depend on deterministic target decoding: EMO improves over Base in all 18 single-seed configurations spanning greedy, top-$p$, and top-$k$ decoding.
For Llama-3.1-8B on MMLU-Pro, the respective gains are $+18.6$, $+21.8$, and $+25.2$ points.
The full method's gains persist under stochastic token selection; this comparison does not isolate the M-step under sampling.

\paragraph{Confidence-constrained execution.}
The hyperparameter sweep tests correction strength and retained token support separately.
On Qwen3-8B/GPQA-D, accuracy stays at $54.5$--$56.6$ for $\mu=1.6$--$3.2$ at $\alpha=0.1$.
At $\mu=2.4$, thresholds $\alpha=0.1$--$0.4$ yield $56.6$--$58.6$, whereas the more permissive $\alpha=0.05$ reduces accuracy to $50.0$ (Appendix~\ref{app:hyperparam}).
Gains persist across the tested strengths but weaken at the most permissive threshold, supporting confidence-constrained execution.
A mask-only ablation would be needed to separate the contrast signal's contribution from the support constraint.


\subsection{Comparison at Matched Generation Budgets}
\label{sec:compute}

\begin{table}[t]
\centering
\caption{\textbf{Matched generation budgets} on MMLU-Pro (accuracy \%, mean$\pm$std over 3 seeds, top-$p$ sampling). iter/$N$: EMO rounds $K$ / target-model generation equivalents $N=2K+1$. \emph{Pass}: oracle coverage over $K+1$ candidates for EMO and $N$ for SC/BoN.}
\label{tab:compute}
\setlength{\tabcolsep}{4pt}
\small
\begin{tabular}{l c ccc cc}
\toprule
Model & iter/$N$ & SC & BoN & EMO & Pass (oracle) SC/BoN & Pass (oracle) EMO \\
\midrule
\multirow{5}{*}{Qwen3-8B}
 & 0/1 & 65.5{\tiny$\pm$1.2} & 65.5{\tiny$\pm$1.2} & 63.9{\tiny$\pm$1.3} & 65.5{\tiny$\pm$1.2} & 63.9{\tiny$\pm$1.3} \\
 & 1/3 & 66.8{\tiny$\pm$1.2} & 67.1{\tiny$\pm$1.0} & \B{67.9{\tiny$\pm$1.1}} & 75.2{\tiny$\pm$0.8} & \B{75.8{\tiny$\pm$1.5}} \\
 & 2/5 & 67.6{\tiny$\pm$0.8} & 68.2{\tiny$\pm$0.7} & \B{69.4{\tiny$\pm$0.6}} & 78.7{\tiny$\pm$0.3} & \B{80.3{\tiny$\pm$1.7}} \\
 & 3/7 & 67.9{\tiny$\pm$0.3} & 68.2{\tiny$\pm$0.6} & \B{70.0{\tiny$\pm$0.8}} & 80.6{\tiny$\pm$0.3} & \B{83.0{\tiny$\pm$2.2}} \\
 & 4/9 & 67.1{\tiny$\pm$1.2} & 68.4{\tiny$\pm$0.5} & \B{70.9{\tiny$\pm$0.7}} & 81.9{\tiny$\pm$0.5} & \B{84.6{\tiny$\pm$1.7}} \\
\midrule
\multirow{5}{*}{Llama-3.1-8B}
 & 0/1 & 45.1{\tiny$\pm$0.6} & 45.1{\tiny$\pm$0.6} & 45.1{\tiny$\pm$1.0} & 45.1{\tiny$\pm$0.6} & 45.1{\tiny$\pm$1.0} \\
 & 1/3 & 50.4{\tiny$\pm$1.5} & 55.1{\tiny$\pm$1.0} & \B{59.4{\tiny$\pm$2.6}} & 62.3{\tiny$\pm$1.1} & \B{66.9{\tiny$\pm$2.1}} \\
 & 2/5 & 49.8{\tiny$\pm$0.2} & 58.0{\tiny$\pm$1.1} & \B{63.7{\tiny$\pm$1.5}} & 68.9{\tiny$\pm$0.4} & \B{76.2{\tiny$\pm$2.1}} \\
 & 3/7 & 51.0{\tiny$\pm$1.1} & 61.0{\tiny$\pm$0.8} & \B{65.7{\tiny$\pm$1.1}} & 73.9{\tiny$\pm$0.1} & \B{80.2{\tiny$\pm$1.3}} \\
 & 4/9 & 51.4{\tiny$\pm$0.8} & 62.5{\tiny$\pm$0.5} & \B{68.2{\tiny$\pm$0.8}} & 76.4{\tiny$\pm$0.2} & \B{83.2{\tiny$\pm$1.1}} \\
\midrule
\multirow{5}{*}{Llama-3.2-3B}
 & 0/1 & 36.5{\tiny$\pm$1.7} & 36.5{\tiny$\pm$1.7} & 35.5{\tiny$\pm$1.3} & 36.5{\tiny$\pm$1.7} & 35.5{\tiny$\pm$1.3} \\
 & 1/3 & 40.5{\tiny$\pm$0.9} & 46.4{\tiny$\pm$0.8} & \B{53.2{\tiny$\pm$1.5}} & 54.0{\tiny$\pm$0.7} & \B{59.9{\tiny$\pm$1.5}} \\
 & 2/5 & 41.2{\tiny$\pm$0.8} & 49.8{\tiny$\pm$1.1} & \B{60.4{\tiny$\pm$1.5}} & 60.8{\tiny$\pm$1.3} & \B{71.5{\tiny$\pm$0.7}} \\
 & 3/7 & 40.9{\tiny$\pm$0.8} & 51.9{\tiny$\pm$0.7} & \B{64.1{\tiny$\pm$0.5}} & 65.1{\tiny$\pm$1.2} & \B{77.9{\tiny$\pm$0.6}} \\
 & 4/9 & 42.0{\tiny$\pm$0.5} & 53.5{\tiny$\pm$0.4} & \B{66.2{\tiny$\pm$0.5}} & 68.5{\tiny$\pm$1.3} & \B{81.6{\tiny$\pm$0.4}} \\
\bottomrule
\end{tabular}
\end{table}

\begin{figure}[t]
\centering
\begin{subfigure}[t]{0.49\textwidth}
\centering
\includegraphics[width=\linewidth]{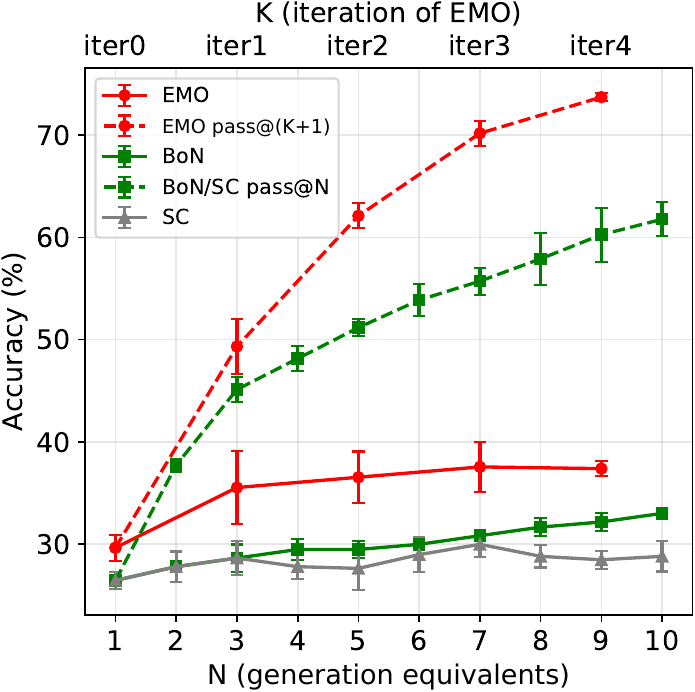}
\caption{Llama-3.1-8B / GPQA-D}
\label{fig:compute-a}
\end{subfigure}
\hfill
\begin{subfigure}[t]{0.49\textwidth}
\centering
\includegraphics[width=\linewidth]{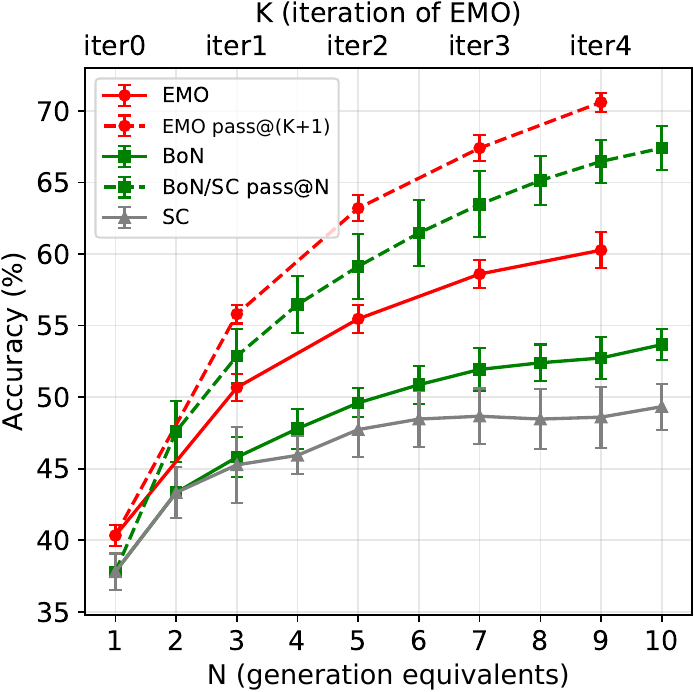}
\caption{Llama-3.2-3B / MATH-500}
\label{fig:compute-b}
\end{subfigure}
\caption{\textbf{Generation-budget scaling} under top-$p$ sampling (mean$\pm$std over three seeds). At $N=2K+1$, EMO has $K+1$ candidates; SC/BoN have $N$ independent samples. Dashed curves show oracle coverage.}
\label{fig:compute}
\end{figure}

\paragraph{Budget accounting.}\sloppy
We compare EMO with SC and judge-selected BoN under top-$p$ sampling, using three seeds (Table~\ref{tab:compute}, Figure~\ref{fig:compute}; full results in Appendix~\ref{app:compute_full}).
SC/BoN share a sample pool and match EMO's decoding settings, per-response token limits, and seeds; BoN also uses the same judge and tie-breaking rule.
The separately generated $N=1$ responses can differ in accuracy despite matching seed identifiers (e.g., $40.3$ for EMO vs.\ $37.8$ for SC/BoN on Llama-3.2-3B/MATH-500), so endpoint gaps compare the complete procedures rather than isolating the effect of later feedback rounds.
One initial generation and two context-conditioned streams per optimization round give EMO@$K$ a budget of $N=2K+1$ target-model generation equivalents, matched to $N$ independent samples for SC/BoN.
The EMO pool contains only $K+1$ candidates because the two streams jointly generate one response per round.
This accounting excludes judge/rewrite calls and response-length differences; total FLOPs and latency are not matched.

\paragraph{Clear advantages on Llama, task-dependent gains on Qwen.}\sloppy
At $N=9$, EMO exceeds BoN on all six Llama tasks: Llama-3.1-8B GPQA-D ($37.4$ vs.\ $32.2$), MATH-500 ($65.6$ vs.\ $60.8$), MMLU-Pro ($68.2$ vs.\ $62.5$), and Llama-3.2-3B GPQA-D ($36.5$ vs.\ $30.0$), MATH-500 ($60.3$ vs.\ $52.7$), MMLU-Pro ($66.2$ vs.\ $53.5$).
It also exceeds SC on all six.
The margin widens on the smaller Llama-3.2-3B, reaching $12.7$ points on MMLU-Pro.
On Llama-3.1-8B/MMLU-Pro, BoN already exceeds SC by $11.1$ points using the same sample pool: candidate availability alone does not ensure successful majority voting.
EMO then raises both coverage and selected accuracy beyond BoN.
For Qwen3-8B, EMO gains $2.5$ points over BoN on MMLU-Pro, but the endpoint means are close on GPQA-D ($60.3$ vs.\ $60.4$) and MATH-500 ($85.0$ vs.\ $84.9$).
At $N=9$, SC is strongest on Qwen3-8B/MATH-500 ($86.1$).
In the two near-ties on Qwen, EMO also lacks an oracle-coverage advantage (Appendix Table~\ref{tab:compute_pass}); its clearest accuracy gains coincide with expanded candidate coverage.

\subsection{Iteration Dynamics and Selection}
\label{sec:iteration_selection}

\paragraph{Early gains with task-dependent returns to further iterations.}
In the sampling experiments, the first two optimization rounds capture roughly $70$--$89\%$ of the final gain in five of six Qwen3-8B and Llama-3.1-8B model--task pairs (Appendix Table~\ref{tab:compute_emopt}); on Llama-3.2-3B the corresponding range is $65$--$81\%$.
On the Llama-3.2-3B/MATH-500 pair in Figure~\ref{fig:compute-b}, $K=2$ reaches $55.5$ accuracy ($76\%$ of the eventual gain over the $40.3$ initial accuracy), and the last two rounds add another $4.8$ points.
Qwen3-8B/GPQA-D is the exception: $K=2$ captures about $46\%$, and accuracy continues to improve through $K=4$.
These trajectories motivate task-dependent iteration budgets, with additional rounds remaining useful where early gains are limited.

\paragraph{Candidate coverage and answer selection impose different limits.}
At $N=9$, EMO's five-candidate oracle coverage exceeds that of nine independent samples in seven of nine settings (Appendix Table~\ref{tab:compute_pass}).
Figure~\ref{fig:compute-a} illustrates the largest coverage advantage, on Llama-3.1-8B/GPQA-D.
EMO already reaches $35.5$ accuracy at $K=1$ ($N=3$), exceeding BoN with ten samples ($33.0$).
Its dashed oracle curve separates further from independent sampling, reaching $73.7$ versus $60.3$ at $N=9$, despite using five rather than nine candidates.
However, the growing gap between EMO's dashed and solid curves exposes a selection bottleneck: from $K=1$ to $K=4$, coverage rises from $49.3$ to $73.7$, while selected accuracy rises only from $35.5$ to $37.4$.
At $N=9$, selected accuracy as a fraction of coverage is $50.7\%$ for EMO versus $53.4\%$ for BoN: the advantage comes from broader coverage despite a lower fraction realized by selection (Appendix~\ref{app:compute_full}).
Figure~\ref{fig:compute-b} shows a different pattern on Llama-3.2-3B/MATH-500: at $N=9$, EMO exceeds BoN in both coverage ($70.6$ vs.\ $66.5$) and selected accuracy ($60.3$ vs.\ $52.7$). The selected-to-covered fraction is also higher ($85.4\%$ vs.\ $79.2\%$), so the accuracy gap accompanies gains in both candidate coverage and the fraction of covered cases selected correctly.
At the other extreme, Llama-3.2-3B/AIME has oracle coverage of only $3.3$ and $6.7$ on the two editions (Appendix Table~\ref{tab:aime_split}), so better selection alone would yield limited gains.
These regimes call for stronger ranking when correct candidates exist and better generation when coverage is low.

\section{Conclusion}
\label{sec:conclusion}

EMO couples question-specific context optimization with logits steering, using each response to refine the context that guides the next generation.
Its logit-space analysis clarifies this connection and identifies conditions for one-round target-KL improvement.
Across eight models and six benchmarks, the results show that model scaling and reasoning distillation leave scope for adaptation, with steering benefits that depend on the model and task.
The broader implication is to treat generation and verification as coupled uses of test-time computation.
Feedback can make further attempts more productive, while reliable selection determines how much of their expanded coverage reaches the final answer.
This perspective motivates inference systems that adjust context revision, decoding strength, and verification effort jointly to the bottleneck encountered during reasoning.

\paragraph{Limitations and future work.}
External judging can introduce misleading feedback or miss correct candidates.
Calibrated ranking and independent checks of reasoning steps or tool outputs could improve verification; retrieval and multiple optimization trajectories could address low candidate coverage.
Distribution access and two model streams limit deployment, and generation-equivalent budgets omit judge/rewrite costs.
Smaller judges, selective M-step activation, and adaptive stopping could reduce overhead; evaluations should measure total tokens, FLOPs, and latency.
Text-only interfaces could use response-level refinement and reranking, subject to separate evaluation.
The one-round guarantee assumes successful context updates and faithful steering over full softmax distributions; it does not certify default settings or monotonic accuracy.
Context-update acceptance checks, adaptive steering strength, and fallback to context-only decoding could improve robustness, using observable signals such as verifier disagreement and score margins.
Evaluation is limited to reasoning benchmarks in two model families.
Alternative judges, multilingual tasks, long-context workflows, and tool-using agents would test transfer and error accumulation over longer adaptation cycles.

\bibliography{2026_conference}
\bibliographystyle{2026_conference}

\appendix
\newpage
\section{Proofs and Method Supplementary Details}
\label{app:method09}

\subsection{Proof of the Unified Objective}
\label{app:objective09}

For each question with reference response $y^\star$, define $Q(y\mid q)=\mathbf 1\{y=y^\star\}$.
Then
\begin{equation}
 \mathrm{KL}(Q\|P_\xi)
 =\sum_{y:Q(y\mid q)>0}Q(y\mid q)\log\frac{Q(y\mid q)}{P_\xi(y\mid q)}
 =-\log P_\xi(y^\star\mid q).
\end{equation}
Averaging over the empirical question distribution proves Lemma~\ref{lem:unified_objective}.
The equality holds independently of which component of $\xi=(\boldsymbol{\theta},C)$ is varied.
For a general fixed target, cross-entropy equals $\mathrm{KL}(Q\|P_\xi)+H(Q)$, and the entropy term is independent of either optimization variable.
Thus the minimizers of the two losses agree within each chosen feasible set; the parameter and context feasible sets need not have the same minimizer.

In the test-time setting, take the question distribution to be concentrated on the current question.
The target objective remains a theoretical reference: the gold response is supplied to neither the judge nor the textual optimizer.
The ideal context optimizer is defined operationally as producing a context with smaller target loss whenever an improving update is assumed to exist.
No differentiable structure on the space of strings is needed for that condition, and the objective identity does not assert that arbitrary textual feedback exactly differentiates NLL.

\subsection{Proof of the Unified Steering Form}
\label{app:steering09}
\label{app:equivalence}

\paragraph{Exact realization of either optimization update.}
At a fixed history, let $\mathbf{z}=\mathbf{z}_\xi(h)$ and $\mathbf{z}^+=\mathbf{z}_{\xi^+}(h)$ be the before/after logits.
Defining $\Delta \mathbf{z}=\mathbf{z}^+-\mathbf{z}$ immediately gives
$\mathrm{Softmax}(\mathbf{z}+\Delta \mathbf{z})=P_{\xi^+}(\cdot\mid h)$.
Equivalently, the log-ratio $\log(P_{\xi^+}/P_\xi)$ differs from $\Delta \mathbf{z}$ only by the scalar difference of log-normalizers.
Softmax ignores this token-independent constant.
For parameter optimization, the difference is $f_{\boldsymbol{\phi}}(C,h)-f_{\boldsymbol{\theta}}(C,h)$; for context optimization it is $f_{\boldsymbol{\theta}}(C',h)-f_{\boldsymbol{\theta}}(C,h)$.
Applying the equality at every history and multiplying the conditional probabilities reproduces the optimized autoregressive distribution.
For multiple optimization steps, the logit differences telescope at each fixed history, so the same statement applies to the entire resulting output change.

\paragraph{The ideal direct-logit gradient.}
For $P=\mathrm{Softmax}(\mathbf{z})$ and a fixed conditional target $Q$,
\begin{align}
 \nabla_P\mathrm{KL}(Q\|P)&=-Q/P,\\
 -\nabla_{\mathbf{z}}\mathrm{KL}(Q\|P)
 &=\mathbf J_P(Q/P)
 =Q-P\sum_v Q(v)=Q-P.
\end{align}
The division is componentwise; finite logits ensure strictly positive $P$.
The chain rule uses $\mathbf J_P^\top=\mathbf J_P$.
Softmax ensures that the updated probabilities remain normalized.
For $\boldsymbol{\delta}^*=Q-P$,
\begin{equation}
 \left.\frac{d}{d\eta}\mathrm{KL}\!\left(Q\|\mathrm{Softmax}(\mathbf{z}+\eta\boldsymbol{\delta}^*)\right)\right|_{\eta=0}
 =\langle P-Q,\boldsymbol{\delta}^*\rangle=-\|\boldsymbol{\delta}^*\|_2^2.
\end{equation}
When $Q\ne P$, differentiability gives strict improvement for all sufficiently small positive $\eta$.
When $Q=P$, the direction is zero and no strict improvement is claimed.
For a point-mass target, the reference token receives $1-P(y^\star)$ and other tokens receive $-P(v)$.

\paragraph{Which equivalence this establishes.}
The additive realization is exact for the output of either optimization process.
The ideal gradient formula concerns direct optimization over logits.
For a differentiable parameter logit map with Jacobian $\mathbf{G}=\partial \mathbf{z}/\partial\boldsymbol{\theta}$, a parameter gradient step instead induces
\begin{equation}
 \mathbf{z}_{\boldsymbol{\theta}^+}-\mathbf{z}_{\boldsymbol{\theta}}=-\eta \mathbf{G}\mathbf{G}^\top(P-Q)+O(\eta^2),
\end{equation}
assuming the map is twice continuously differentiable near the update.
This follows by applying the chain rule to the parameter loss and Taylor-expanding the updated logits.
For a training loss over multiple examples and histories, the analogous statement uses stacked logits, their Jacobian, and the full weighted loss gradient.
Textual updates induce finite differences without requiring a string derivative.
These observations establish a shared logit-space representation without identifying different optimizers' actual directions with $Q-P$.

\paragraph{Connection to reinforcement learning.}
At a fixed history, let $P$ be a positive reference policy, $A$ a fixed token advantage, and $\beta>0$.
For the unrestricted policy problem
$\max_\pi\{\langle\pi,A\rangle-\beta^{-1}\mathrm{KL}(\pi\|P)\}$,
define $Z=\sum_vP(v)e^{\beta A(v)}$ and $\pi^*(v)=P(v)e^{\beta A(v)}/Z$.
The objective equals $\beta^{-1}\log Z-\beta^{-1}\mathrm{KL}(\pi\|\pi^*)$, so its unique maximizer is $\pi^*$.
It therefore has the additive form $\pi^*=\mathrm{Softmax}(\mathbf{z}+\beta A)$.
This gives the policy-improvement connection under the stated fixed-history objective; it does not identify arbitrary parameterized RL updates with direct target-KL gradient descent.

\subsection{From the Context Contrast to the Steering Vector}
\label{app:contrast09}

The construction follows the probability-gradient-to-logit chain rule in SVDecode~\citep{hu2026distribution}.
Hold $P_{\mathrm{base}}$ fixed and differentiate at $P_{\mathrm{opt}}=\mathrm{Softmax}(\mathbf{z}^{\mathrm{opt}})$:
\begin{align}
 \nabla_{P_{\mathrm{opt}}}\mathrm{KL}(P_{\mathrm{opt}}\|P_{\mathrm{base}})
 &=\log(P_{\mathrm{opt}}/P_{\mathrm{base}})+\mathbf 1,\\
 \nabla_{\mathbf{z}^{\mathrm{opt}}}\mathrm{KL}(P_{\mathrm{opt}}\|P_{\mathrm{base}})
 &=\mathbf J_{P_{\mathrm{opt}}}\log(P_{\mathrm{opt}}/P_{\mathrm{base}}).
\end{align}
The constant term vanishes because $\mathbf J_{P_{\mathrm{opt}}}\mathbf 1=0$.
This proves Eq.~(\ref{eq:mstep_delta}); the equality defines the implemented vector, while its task-alignment interpretation remains conditional.
Writing $r(v)=\log P_{\mathrm{opt}}(v)-\log P_{\mathrm{base}}(v)$ gives the centered expression
\begin{equation}
 \delta_t(v)=P_{\mathrm{opt}}(v)
 \left(r(v)-\sum_wP_{\mathrm{opt}}(w)r(w)\right).
 \label{eq:centered_contrast09}
\end{equation}
Thus a positive correction requires a log-ratio above its optimized-distribution average.
The condition $P_{\mathrm{opt}}(v)>P_{\mathrm{base}}(v)$ alone is insufficient.
The expression is evaluated in linear vocabulary time using log-softmax values, without forming a dense Jacobian or backpropagating through model parameters.

For the unmasked path $P(u)=\mathrm{Softmax}(\mathbf{z}^{\mathrm{opt}}+u\boldsymbol{\delta}_t)$,
\begin{equation}
 \left.\frac{d}{du}\mathrm{KL}(P(u)\|P_{\mathrm{base}})\right|_{u=0}
 =\|\boldsymbol{\delta}_t\|_2^2.
\end{equation}
This explains the positive sign for local extrapolation in Eq.~(\ref{eq:mstep_delta}).
The negative-gradient convention in SVDecode would give the opposite local change for positive strength.
Neither sign determines target improvement without a target-alignment condition.
In particular, the log-ratio in Eq.~(\ref{eq:mstep_delta}) is an observable contrast, not an identity or a generally consistent estimate of $Q/P_{\mathrm{opt}}$ in Lemma~\ref{lem:unified_steering}.
After masking, the direction is $\hat{\boldsymbol{\delta}}_t$ and all target-alignment reasoning must use this actual vector.

\paragraph{Complementarity and sequential dependence.}
The steering operator supplies an execution rule, whereas feedback and context optimization supply task information.
For any fixed direction $\mathbf{d}$, the derivative of target KL at zero is $\langle P-Q,\mathbf{d}\rangle$, so its usefulness depends on the target.
EMO obtains its direction only after $C_{k+1}$ has been estimated; the pair $(C_k,C_{k+1})$ then determines the contrast at each prefix.
The resulting steered response drives the next context update.
This is the ordered, complementary relationship underlying the E/M terminology.
Classical EM monotonicity would require a specified latent-variable objective and corresponding expectation and maximization steps, which are not supplied by this structural correspondence alone.

\subsection{Proof of Conditional One-Round Improvement}
\label{app:theory09}

Fix a question and condition on the realized contexts $C_k,C_{k+1}$, including any randomness in the E-step.
Abbreviate the fixed sequence target by $Q$; it may be the point-mass reference target from Lemma~\ref{lem:unified_objective} or a distribution over multiple admissible responses.
Use a finite vocabulary $\mathcal V$, maximum horizon $T$, finite logits at nonterminal histories, and a fixed threshold $\alpha$.
After EOS, pad with shared deterministic EOS transitions and zero directions; their KL contributions are zero.
All distributions use this common finite-horizon convention.
The distributions below are full softmax distributions before temperature scaling or decoder truncation.
Let $\widetilde P_\mu$ denote the steered sequence distribution, so $\widetilde P_0=P_{\boldsymbol{\theta},C_{k+1}}$.
For a fixed history, $\hat{\boldsymbol{\delta}}_t(h)$ depends only on the contexts and $\alpha$, not on $\mu$.
Generated prefixes can change with $\mu$, but the direction at any specified prefix remains fixed; the proof differentiates this conditional distribution field, not a sampled decoding trajectory.
Only histories with positive probability under $Q$ contribute, so target conditionals at other histories may be defined arbitrarily.

\paragraph{The two conditions.}
The oracle E-step condition is
\begin{equation}
 \mathrm{KL}(Q\|P_{\boldsymbol{\theta},C_{k+1}})<\mathrm{KL}(Q\|P_{\boldsymbol{\theta},C_k}).
 \label{eq:estep_improvement}
\end{equation}
Faithful steering is made precise by requiring positive aggregate alignment of the actual masked direction with the ideal target descent direction:
\begin{equation}
 \Gamma=\sum_{t=1}^T\mathbb E_{Y_{<t}\sim Q}
 \left\langle Q_t(\cdot\mid h)-P_{\mathrm{opt}}(\cdot\mid h),\hat{\boldsymbol{\delta}}_t(h)\right\rangle>0.
 \label{eq:faithful_steering09}
\end{equation}
This condition includes the threshold and finite penalty from Eq.~(\ref{eq:confidence_mask}).
It permits some histories to have unfavorable directions, provided their target-weighted aggregate is favorable.

\paragraph{Conditional derivative and sequence composition.}
For $\widetilde P_\mu(\cdot\mid h)=\mathrm{Softmax}(\mathbf{z}_t^{\mathrm{opt}}+\mu\hat{\boldsymbol{\delta}}_t)$, differentiation gives
\begin{equation}
 \left.\frac{d}{d\mu}\mathrm{KL}(Q_t\|\widetilde P_\mu(\cdot\mid h))\right|_{\mu=0}
 =-\langle Q_t-P_{\mathrm{opt}},\hat{\boldsymbol{\delta}}_t\rangle.
\end{equation}
Expanding the sequence log-ratio into conditional log-ratios gives the KL chain rule:
\begin{equation}
 \mathrm{KL}(Q\|\widetilde P_\mu)
 =\sum_{t=1}^T\mathbb E_{Y_{<t}\sim Q}
 \mathrm{KL}\!\left(Q_t(\cdot\mid h)\|\widetilde P_\mu(\cdot\mid h)\right).
\end{equation}
The vocabulary and horizon are finite, so differentiation may be interchanged with these finite sums.
The derivative of sequence KL at $\mu=0$ is the negative of Eq.~(\ref{eq:faithful_steering09}), hence strictly negative.
By differentiability there is $\varepsilon>0$ such that every $0<\mu<\varepsilon$ gives
\begin{equation}
 \mathrm{KL}(Q\|\widetilde P_\mu)
 <\mathrm{KL}(Q\|P_{\boldsymbol{\theta},C_{k+1}})
 <\mathrm{KL}(Q\|P_{\boldsymbol{\theta},C_k}),
\end{equation}
where the second inequality is Eq.~(\ref{eq:estep_improvement}).
This proves Theorem~\ref{thm:monotonicity}.

\paragraph{A quantitative sufficient step-size interval.}
The local result can be made explicit using the curvature of sequence KL along the fixed steering direction.
Write $F(\mu)=\mathrm{KL}(Q\|\widetilde P_\mu)$ and $\mathbf d_t(h)=\hat{\boldsymbol{\delta}}_t(h)$.
Differentiating the conditional log-normalizer twice gives
\begin{align}
 F'(\mu)
 &=\sum_{t=1}^{T}\mathbb E_{Y_{<t}\sim Q}
 \left\langle\widetilde P_\mu(\cdot\mid h)-Q_t(\cdot\mid h),\mathbf d_t(h)\right\rangle,
 \label{eq:sequence_first_derivative09}\\
 F''(\mu)
 &=\sum_{t=1}^{T}\mathbb E_{Y_{<t}\sim Q}
 \operatorname{Var}_{V\sim\widetilde P_\mu(\cdot\mid h)}[d_t(h,V)]\geq0,
 \label{eq:sequence_curvature09}
\end{align}
where $V$ is a next-token random variable and $d_t(h,v)$ is the $v$th entry of $\mathbf d_t(h)$.
Thus $F$ is convex along this one-dimensional path, even though context optimization itself is discrete and need not be convex.
Define the range of the constrained correction and its target-weighted squared-range bound by
\begin{equation}
 \rho_t(h)=\max_{v\in\mathcal V}d_t(h,v)-\min_{v\in\mathcal V}d_t(h,v),
 \qquad L=\frac14\sum_{t=1}^{T}\mathbb E_{Y_{<t}\sim Q}\rho_t(h)^2.
 \label{eq:steering_curvature_bound09}
\end{equation}
For a random variable supported on $[a,b]$, its variance is at most $(b-a)^2/4$: subtracting the midpoint gives a second moment bounded by this quantity, and variance is no larger than that second moment.
Consequently, $0\leq F''(\mu)\leq L$ for every finite $\mu$.
The bound is finite because the vocabulary, horizon, logits, and penalty are finite.
Moreover, $\Gamma>0$ implies $L>0$: otherwise all directions at target-supported histories would be tokenwise constant, making their inner products with $Q_t-P_{\mathrm{opt}}$ zero.
Integrating the curvature bound and using $F'(0)=-\Gamma$ yields
\begin{equation}
 F(\mu)\leq F(0)-\mu\Gamma+\frac{L\mu^2}{2},
 \qquad \mu\geq0.
 \label{eq:quantitative_improvement09}
\end{equation}
Hence every $0<\mu<2\Gamma/L$ strictly improves over the unsteered updated-context distribution.
At $\mu=\Gamma/L$, the bound guarantees a decrease of at least $\Gamma^2/(2L)$.
Combined with Eq.~\ref{eq:estep_improvement}, this also gives improvement over the preceding context.
This sufficient interval depends on the unknown target through $\Gamma$ and $L$; it is an analytical guarantee, not an implementable tuning prescription or a certificate for the empirical default.

Convexity also clarifies the role of the alignment assumption.
If $\Gamma\leq0$, then $F'(0)\geq0$ and $F(\mu)\geq F(0)+\mu F'(0)\geq F(0)$ for all $\mu\geq0$.
Thus, for fixed contexts and threshold, positive aggregate alignment is necessary and sufficient for some positive strength to improve target KL over $P_{\boldsymbol{\theta},C_{k+1}}$ along this path.
This statement does not compare different context updates, thresholds, or steered rounds.

\paragraph{Role of the confidence constraint.}
The alignment contribution at a history is
\begin{equation}
 \langle Q_t-P_{\mathrm{opt}},\hat{\boldsymbol{\delta}}_t\rangle
 =\sum_{v:\mathbb I_t(v)=1}(Q_t(v)-P_{\mathrm{opt}}(v))\delta_t(v)
 +\lambda\sum_{v:\mathbb I_t(v)=0}(Q_t(v)-P_{\mathrm{opt}}(v)).
\end{equation}
Changing $\alpha$ changes both sets and can therefore change the sign of this contribution.
Writing $S_t=\{v:\mathbb I_t(v)=1\}$, the excluded-token contribution is $\lambda[Q_t(S_t^c)-P_{\mathrm{opt}}(S_t^c)]$.
Because $\lambda<0$, this term helps alignment when excluded tokens carry less target mass than model mass, and hurts when they carry more.
Relative model confidence is therefore a plausibility criterion, not a guarantee of target correctness.
The claim applies to the masked direction at the chosen $\alpha$, not merely the raw contrast.
The sufficient interval for $\mu$ depends on these directions and the unknown target; the proof does not certify the empirical default or yield a test-time tuning rule.
The oracle E-step condition does not by itself imply Eq.~(\ref{eq:faithful_steering09}).
Across rounds, the comparison distribution above is the preceding unsteered context distribution, not the preceding round's steered distribution.
These distinctions prevent a conditional one-round result from being read as unconditional iterative convergence or an accuracy guarantee.

\paragraph{Why the two assumptions are distinct.}
A one-step, three-token example shows that context improvement need not imply useful steering.
Take $Q=(1,0,0)^\top$, $P_{\mathrm{base}}=(0.2,0.3,0.5)^\top$, and $P_{\mathrm{opt}}=(0.25,0.65,0.10)^\top$.
The context update improves target KL from $-\log0.2$ to $-\log0.25$.
At $\alpha=0.1$, all tokens pass the confidence threshold, and Eq.~\ref{eq:centered_contrast09} gives
$\boldsymbol{\delta}\approx(-0.0436,0.2443,-0.2007)^\top$.
Nevertheless, $\Gamma=\langle Q-P_{\mathrm{opt}},\boldsymbol{\delta}\rangle\approx-0.1714<0$.
The context update increases the correct token's probability, but the contrast favors a competing token strongly enough that further extrapolation raises target KL.
This constructed example illustrates the logical need for Assumption~\ref{ass:steering}; it is not an experimental result.

\subsection{Non-Regression under Judge Selection}
\label{app:selection09}

Store the score $a_k=\mathcal J(q,R_k)$ once for each response, and select the earliest candidate attaining the largest score.
Then the selected score is at least every stored score, including $a_0$.
Adding candidates can only increase or preserve this maximum.
This proves the non-regression observation when $\mathrm{Score}$ denotes the stored judge score.
It does not establish non-regression of rule-based correctness, which can disagree with judge ranking.

\subsection{Implementation Conventions}
\label{app:protocol09}

The steering strength defaults to $\mu=2.4$ and the threshold to $\alpha=0.1$; the finite penalty is fixed at $\lambda=-100$.
An excluded token therefore receives an additive logit change $\mu\lambda$; this is distinct from overwriting its final logit with $-100$.
The same prefix is appended to separate KV caches for $C_k$ and $C_{k+1}$.
Contrast and masking are computed before any sampling temperature or top-$p$ truncation.
No context is carried from one test question to the next.
In the stated update, rejecting a rewrite sets $C_{k+1}=C_k$ and makes the raw contrast zero.
The confidence penalty in Eq.~\ref{eq:confidence_mask} can still affect decoding when $\mu>0$; zero contrast alone does not remove this constraint.
In contrast, setting $\mu=0$ removes both the contrast correction and the penalty and recovers context-only decoding.
This distinguishes a rejected rewrite from explicitly disabling the M-step.

With $K=4$ refinement rounds, the initial response and four steered responses form five candidates.
Each M-step uses two target-model context streams, giving nine generation equivalents including initialization.
This counts streams rather than exact FLOPs or latency and excludes external judge, backward, and rewrite cost as well as differences in context and response lengths.

\section{Experiments Supplementary Details}
\label{app:experiments}

\subsection{Datasets}
\label{app:datasets}

We evaluate on six benchmarks spanning free-form mathematical reasoning and multiple-choice scientific knowledge, summarized in Table~\ref{tab:app_datasets}.
All datasets are preprocessed into a unified schema $\{\texttt{id}, \texttt{task}, \texttt{question}, \texttt{choices}, \texttt{answer}\}$, where \texttt{task} $\in$ \{\texttt{math}, \texttt{math\_boxed}, \texttt{mc\_boxed}\} determines both the prompt template and the rule-based scoring rule (Appendix~\ref{app:settings}).

\begin{table}[h]
\centering
\caption{\textbf{Benchmarks.} Task format, size, and answer form. MC = multiple choice. MMLU-Pro is our stratified 560-question subset (14 categories $\times$ 40, sampling seed 42).}
\label{tab:app_datasets}
\small
\begin{tabular}{l r l l l}
\toprule
Benchmark & $n$ & Task type & Answer form & Source \\
\midrule
AIME-2025 & 30 & math\_boxed & integer & \citeyearpar{Balunovic2025MathArenaEL} \\
AIME-2026 & 30 & math\_boxed & integer & \citeyearpar{Balunovic2025MathArenaEL} \\
GPQA-Diamond & 198 & MC (A--D) & option letter & \citeyearpar{Rein2023GPQAAG} \\
GSM8K & 1{,}319 & math & \texttt{\#\#\#\# number} & \citeyearpar{Cobbe2021TrainingVT} \\
MATH-500 & 500 & math\_boxed & LaTeX / numeric & \citeyearpar{lightman2023lets} \\
MMLU-Pro (sampled) & 560 & MC (A--J) & option letter & \citeyearpar{Wang2024MMLUProAM} \\
\bottomrule
\end{tabular}
\end{table}

\paragraph{AIME 2025/2026.}
The American Invitational Mathematics Examination is a competition-mathematics benchmark with 30 problems per edition; answers are integers in $[0,999]$.
We use the 2025 and 2026 editions as separate test sets (30 problems each) and merge them in the main table for noise reduction ($n=30$ per year yields a granularity of $3.3$ points).

\paragraph{GPQA-Diamond.}
The Diamond subset of the Graduate-Level Google-Proof Q\&A benchmark consists of 198 multiple-choice science questions written and validated by domain experts (physics, chemistry, biology).
Each question has four options; we score the letter inside \texttt{boxed\{\}\,} (falling back to an \texttt{Answer: X} pattern).

\paragraph{GSM8K.}
Grade School Math 8.5K contains 1{,}319 grade-school word problems in its test split.
Following the original format, models must write the final numerical answer on the last line as \texttt{\#\#\#\# <number>}; scoring is exact numeric match after stripping commas and decimal points.
GSM8K is the most saturated benchmark in our suite (three models exceed $89\%$ baseline accuracy), which bounds the observable headroom.

\paragraph{MATH-500.}
We use the 500-problem evaluation subset of the MATH competition dataset~\citep{Hendrycks2021MeasuringMP}, following the held-out split of \citet{lightman2023lets}.
It covers seven subjects with LaTeX-form answers (e.g., $\frac{1}{2}$, $-\sqrt{3}$, $5r^5$).
This benchmark stresses symbolic answer matching; our scoring pipeline handles numeric tolerance, LaTeX normalization, and SymPy-based symbolic equivalence (Appendix~\ref{app:settings}).

\paragraph{MMLU-Pro (stratified subset).}
The full MMLU-Pro test set contains 12{,}032 questions across 14 subject categories with up to 10 options each.
For compute efficiency we evaluate on a category-balanced subset of 560 questions: within each of the 14 categories we sample 40 questions uniformly at random with a fixed seed (\texttt{seed=42}), giving equal representation to all categories.
The sampling procedure is implemented in \texttt{datasets/process\_mmlu\_pro.py} and is exactly reproducible:
\begin{quote}
\small
\texttt{rng = random.Random(42)} \\
\texttt{sampled = [rng.sample(cat\_items[c], 40) for c in categories]}
\end{quote}
Prompts are constructed from the question stem and the full option block, with the output format \texttt{Answer: \$LETTER\$} and a step-by-step instruction.
The three-task sampling analysis (top-$p$, 3 seeds) uses GPQA-Diamond, MATH-500, and MMLU-Pro to cover science MC, math free-form, and general-knowledge MC; we exclude AIME ($n=30$, high variance) and GSM8K (saturated) from that analysis.

\subsection{Model Selection}
\label{app:models}

The eight evaluated models are listed in Table~\ref{tab:app_models}, organized around four analysis axes with Qwen3-8B as the anchor.
All models are loaded from their official Hugging Face weights with default configuration (\texttt{dtype=auto}).

\begin{table}[h]
\centering
\caption{\textbf{Evaluated models.} ``Thinking'' marks models with an explicit long chain-of-thought mode. The selection spans scale (3B--32B), three Qwen generations, two families, and thinking/non-thinking modes.}
\label{tab:app_models}
\small
\setlength{\tabcolsep}{4pt}
\renewcommand{\arraystretch}{1.15}
\begin{tabular}{@{}p{2.6cm} p{4.6cm} p{3.4cm} c@{}}
\toprule
Model & HF checkpoint & Axis probed & Thinking \\
\midrule
Llama-3.2-3B-Instruct & \texttt{meta-llama/ Llama-3.2-3B-Instruct} & scale (low) & \xmark \\
Llama-3.1-8B-Instruct & \texttt{meta-llama/ Llama-3.1-8B-Instruct} & family, non-thinking & \xmark \\
Qwen2.5-7B-Instruct & \texttt{Qwen/ Qwen2.5-7B-Instruct} & generation (prev), non-thinking & \xmark \\
Qwen3-8B & \texttt{Qwen/Qwen3-8B} & anchor & \cmark \\
Qwen3.5-9B & \texttt{Qwen/Qwen3.5-9B} & generation (newest) & \cmark \\
Qwen3-14B & \texttt{Qwen/Qwen3-14B} & scale (mid) & \cmark \\
Qwen3-32B & \texttt{Qwen/Qwen3-32B} & scale (high) & \cmark \\
DeepSeek-R1-0528- Qwen3-8B & \texttt{deepseek-ai/ DeepSeek-R1-0528- Qwen3-8B} & reasoning distillation & \cmark \\
\bottomrule
\end{tabular}
\end{table}

\paragraph{Scale.}
Qwen3-8B/14B/32B and Llama-3.2-3B/Llama-3.1-8B probe whether test-time gains shrink as models grow (main-table $\overline{\Delta}$ analysis).
Mean gains within Qwen3 are non-monotonic ($+9.42$, $+7.44$, and $+8.68$ points for 8B, 14B, and 32B), showing that parameter scale alone does not determine adaptation benefit.
The larger gains on AIME than on near-saturated GSM8K further indicate task dependence within each scale.

\paragraph{Generation.}
Qwen2.5-7B-Instruct $\to$ Qwen3-8B $\to$ Qwen3.5-9B form a three-generation line within one family.
Every generation still gains $+6$ to $+10$ points on average after test-time optimization, answering the question of whether newer base models obsolete per-question adaptation: they do not.

\paragraph{Family.}
Qwen3-8B vs.\ Llama-3.1-8B-Instruct compares families at matched parameter scale, while retaining differences in architecture and training.
Llama-3.1-8B gains $22.4$ points on MATH-500, demonstrating applicability beyond the Qwen family; the comparison does not isolate an individual training choice.

\paragraph{Thinking mode.}
Qwen3-8B natively supports a thinking mode (explicit chain of thought), whereas Qwen2.5-7B-Instruct and Llama-3.1-8B-Instruct do not.
Non-thinking models gain more over Base on MATH-500 ($+9.2$/$+22.4$ vs.\ $+8.0$), showing that the full framework also benefits models without a native thinking mode.
Because the checkpoints differ, these gains do not establish that steering substitutes for thinking; on Llama-3.1-8B, TextGrad already accounts for most of the MATH-500 improvement.

\paragraph{Reasoning distillation.}
DeepSeek-R1-0528-Qwen3-8B was obtained by post-training Qwen3-8B-Base on chain-of-thought traces distilled from DeepSeek-R1-0528~\citep{deepseek2025r10528distill}.
EMO improves its own Base on all five task groups ($+6.02$ points on average), demonstrating that this distilled checkpoint retains scope for test-time adaptation.
Its mean M-step increment over TextGrad is $-0.64$, although MMLU-Pro improves by $1.2$ points; thus the benefits of the full adaptation loop and of steering must be distinguished.
A possible explanation is that distillation changes the conditional distributions on which steering operates.
In Eq.~\ref{eq:context_displacement}, both the softmax Jacobian $\mathbf J_{P_{\mathrm{opt}}}$ and the context-induced displacement $\Delta\mathbf{z}_t$ depend on the checkpoint; changes to either can alter the steering vector even with fixed $\mu$ and $\alpha$.
Our Qwen3-8B comparison is not a controlled pre-/post-distillation pair, so this remains a mechanistic hypothesis.
Testing it would require matched checkpoints, contexts, prefixes, and decoding budgets, measuring conditional entropy, context-induced logit displacements, and reference-token probability changes.
The current accuracy results do not establish reduced entropy, overconfidence, or reinforcement-learning saturation.
For this model we report the v2 rerun with an extended generation budget (16{,}384 tokens), as its verbose thinking chains saturate the default 8{,}192-token cap (Appendix~\ref{app:settings}).

\subsection{Experimental Settings}
\label{app:settings}

\paragraph{Metrics.}
\emph{Final accuracy} is rule-based and gold-answer-based; the judge score is never used for reporting, only as the internal optimization signal.
\begin{itemize}[leftmargin=1.5em,itemsep=1pt,topsep=2pt]
\item \textbf{math (GSM8K)}: extract the \texttt{\#\#\#\# <number>} final line; exact numeric match after removing commas and decimal points.
\item \textbf{math\_boxed (MATH-500, AIME)}: extract the content of the final \texttt{\textbackslash boxed\{\}\,}; equality is decided by a four-stage cascade---(i) numeric comparison with $10^{-6}$ tolerance (resolving $\frac{1}{2}$ vs.\ $0.5$), (ii) normalized string match (LaTeX normalization: stripping \texttt{\textbackslash left}/\texttt{\textbackslash right}, \texttt{\textbackslash tfrac}$\to$\texttt{\textbackslash frac}, spacing), (iii) SymPy symbolic equivalence via \texttt{parse\_latex}+\texttt{sympify} combinations, and (iv) a non-trivial-substring fallback (shorter side $\geq 2$ characters and length ratio $<3$).
\item \textbf{mc\_boxed (GPQA, MMLU-Pro)}: the letter inside \texttt{boxed\{\}\,}, falling back to an \texttt{Answer: X} regular expression; exact letter match with gold (options A--J supported).
\end{itemize}
\textbf{judge-select@k}: the reported EMO accuracy looks at the first $k{+}1$ rounds (iter0..iter$k$) and selects the round with the highest judge score (ties broken toward the earlier round), reporting its rule-based correctness.
The main tables use judge-select@4 over five candidates: the initial response and four optimized responses.
\textbf{pass@$N$} is the oracle upper bound: the question counts as solved if any of the first $N$ attempts is correct; it is reported only for reference.

\paragraph{Hyperparameters.}
All experiments use four optimization rounds after the initial generation, with one E-step sample and one M-step sample per round, and the shared initial context \emph{``You are a helpful assistant.''}
Steering defaults: total strength $\mu{=}2.4$, confidence threshold $\alpha{=}0.1$, per-step steering recomputation, and penalty value $-100$ on masked logits.
The generation budget (\texttt{max\_new\_tokens}) is 8{,}192 for GPQA/MATH-500/MMLU-Pro/GSM8K and 12{,}288 for AIME in the greedy main experiments, with the DeepSeek-R1 v2 rerun using 16{,}384 (Appendix~\ref{app:models}); the top-$p$ sampling experiments use 24{,}576 for Qwen3-8B (thinking chains are long) and 12{,}288 for the Llama models.
The top-$p$ sampling experiments use temperature $0.6$, top-$p$ $0.95$, top-$k$ $0$, and seeds $\{0,1,2\}$ with full-chain seeding (\texttt{torch.manual\_seed} at all four sampling sites); each seed also keys the response cache so that seeds do not share generations.
A hyperparameter sweep on Qwen3-8B/GPQA (Appendix~\ref{app:hyperparam}) shows gains are insensitive to the total strength ($+5.6$ to $+7.6$ points across $\mu \in [1.6, 3.2]$) and robust for $\alpha \in [0.1, 0.4]$, while $\alpha{=}0.05$ collapses ($+1.0$) because the confidence constraint becomes too loose.
Table~\ref{tab:app_settings} consolidates all decoding and generation settings per experiment.

\begin{table}[h]
\centering
\caption{\textbf{Consolidated experimental settings.} Generation budget (\texttt{max\_new\_tokens}) is per response; the M-step runs a batch of two contexts per decoding step. The bottom block lists judge scoring and E-step decoding settings.}
\label{tab:app_settings}
\small
\setlength{\tabcolsep}{1.0pt}
\begin{tabular}{l llll}
\toprule
Experiment & Decoding & temp & top-$p$ / top-$k$ & \texttt{max\_new\_tokens} \\
\midrule
Greedy main (8 models) & greedy & -- & -- / -- & 8{,}192; 12{,}288 (AIME); 16{,}384 (DR1-v2) \\
top-$p$ main (Qwen3-8B) & sampling & 0.6 & 0.95 / 0 & 24{,}576 \\
top-$p$ main (Llama-8B/3B) & sampling & 0.6 & 0.95 / 0 & 12{,}288 \\
top-$k$ ablation & sampling & 0.6 & -- / 20 & as top-$p$ main (seed 0 only) \\
BoN/SC pool & sampling & 0.6 & 0.95 / 0 & as top-$p$ main, $n{=}10$/question \\
\midrule
\multicolumn{5}{l}{\textit{Judge (Qwen3.5-122B-A10B FP8, shared by all experiments)}} \\
\midrule
\multicolumn{2}{l}{Scoring} & 0 (greedy) & \multicolumn{2}{l}{cached; \texttt{max\_tokens}=2{,}048; thinking off} \\
\multicolumn{2}{l}{E-step calls (greedy main)} & 0 (greedy) & \multicolumn{2}{l}{caching off; thinking off} \\
\multicolumn{2}{l}{E-step calls (sampling experiments)} & 0.7 & \multicolumn{2}{l}{top-$p$ 0.95; caching off; thinking off} \\
\multicolumn{2}{l}{Server} & \multicolumn{3}{l}{vLLM, TP=4, \texttt{max\_model\_len}=32{,}768, CUDA graphs on} \\
\bottomrule
\end{tabular}
\end{table}

Sampling experiments use seeds $\{0,1,2\}$ with full-chain seeding; every EM run uses four optimization rounds after the initial generation, with one E-step and one M-step sample per round and the initial context \emph{``You are a helpful assistant.''}; evaluation batch sizes (questions per forward batch) are chosen per model and task to fit GPU memory under the worst-case two-context KV budget.

\paragraph{Judge deployment.}
The judge is Qwen3.5-122B-A10B (FP8), served with vLLM (tensor parallelism 4, max context 32{,}768, CUDA graphs enabled).
In the greedy main experiments, the target model and all judge-side calls---scoring, feedback, backward, and context rewriting---use greedy decoding.
In the sampling experiments, judge scoring remains greedy with disk caching, while feedback, backward, and context rewriting use temperature $0.7$, top-$p$ $0.95$, and no caching.
Scoring and E-step calls use separate engine instances.
Both engines disable the judge's thinking mode.
The judge never sees the gold answer; its prompt requires it to independently solve the problem, score 0--10, and---critically---produce feedback \emph{without revealing or hinting at the final answer} (numbers, letters, or expressions), which keeps the E-step gradient estimate free of answer leakage.
An additional regex filter rejects any rewritten context that leaks a concrete answer, retaining the previous context in that case.

\paragraph{Answer-extraction fix and rescoring.}
An early version of the \texttt{math\_boxed} scorer hard-extracted digits from non-numeric boxed answers (e.g., $7\pi \to 7$), systematically under-reporting MATH-500 accuracy.
After fixing the extraction cascade, we rescored \emph{all} affected stored responses with the corrected rule (flips occur only on MATH-500; GPQA/MMLU/GSM8K/AIME are unchanged), and all numbers in this paper use the rescored values.
The rescored greedy baseline of Qwen3-8B on MATH-500 ($75.0$) is consistent with the independently run sampling baseline ($79.6{\pm}0.4$ at temperature 0.6), validating the fix.

\paragraph{Degenerate-generation handling (Llama).}
Llama-3.1/3.2 occasionally enter repetitive loops, filling the entire generation budget with a short cycle.
Beyond the standard exact-cycle detector (10--30 token loops), we add an information-saturation detector: once a generation exceeds 4{,}000 tokens, it is terminated early if the novel-token rate over each of the last three 200-token windows falls below $0.10$ relative to a 400-token baseline window.
On a calibration set the detector hits $97\%$ of degenerate chains with $0\%$ false positives on normal chains and speeds up MATH-500 by ${\sim}12\times$.
Degenerate responses are almost always incorrect, so early termination has negligible effect on accuracy.

\paragraph{Two-arm protocol.}
The TextGrad (no-M-step) arm shares its baseline and iter-0 responses with the EMO arm by construction (\texttt{reuse-baseline/iter0-from}), so both arms start from an identical state and diverge only from iteration 1 onward.
In the sampling experiments the no-M-step arm additionally reuses the same sharding, guaranteeing a 1:1 question match across arms.

\paragraph{Compute.}
Each optimization round generates one response using two context-conditioned model streams. Including the initial generation, EMO@$k$ uses $N=2k+1$ target-model generation equivalents and produces $k+1$ candidates. Section~\ref{sec:compute} matches this budget to $N$ independent samples for SC/BoN; the accounting excludes judge/rewrite calls, response-length differences, and latency.
Judge throughput is not a bottleneck: with CUDA graphs, the 122B FP8 judge sustains ${\sim}22$ concurrent scoring workers per 4-GPU instance.

\paragraph{Self-consistency and Best-of-$N$ baselines.}
The SC/BoN comparison shares the sampling pool of the top-$p$ experiments: for each question we draw $n{=}10$ i.i.d.\ samples at temperature $0.6$, top-$p$ $0.95$, with the same per-model generation budget and the same seeds $\{0,1,2\}$ as the EMO sampling runs.
\textbf{BoN@$N$} selects, among the first $N$ samples, the one with the highest judge score (the same scoring engine and prompt as EMO's judge-select, ties broken toward the earlier sample) and reports its rule-based correctness.
\textbf{SC@$N$} takes a majority vote over the extracted final answers of the first $N$ samples (the same extraction rules as the main scoring pipeline); a sample whose answer cannot be extracted abstains, and if all $N$ abstain the first sample's rule-based verdict is used.
On a vote tie, the answer with the highest total judge score among the tied candidates is selected.
\textbf{pass@$N$} is the oracle upper bound shared by SC and BoN (any of the first $N$ samples correct).

\subsection{Full Experimental Results}
\label{app:full_results}

\subsubsection{AIME-2025 and AIME-2026 Reported Separately}
\label{app:aime_split}


Table~\ref{tab:aime_split} reports AIME-2025 and AIME-2026 separately (the main paper merges the two 30-problem sets).
Because each year has only 30 problems, one problem corresponds to $3.3$ points; small differences should be interpreted in light of this coarse resolution.
EMO more than doubles Qwen3-8B's Base accuracy on both editions ($16.7\!\to\!43.3$ on 2025, $23.3\!\to\!50.0$ on 2026), whereas the Llama models remain near zero.
For Llama-3.2-3B, even pass@5 reaches only $3.3$ and $6.7$ on the two editions.
Thus a perfect selector over these pools would still solve few questions: the principal observed limitation is candidate generation.
Low baseline accuracy provides headroom but does not ensure that four feedback-guided updates can access correct solutions.

\begin{table}[h]
\centering
\caption{\textbf{AIME-2025 and AIME-2026 reported separately} (accuracy \%, greedy; 30 problems per year, so 1 problem $= 3.3$ points). pass@5 is the oracle upper bound over the 5 EMO rounds. The main-paper Table~\ref{tab:main} reports the average of the two years.}
\label{tab:aime_split}
\setlength{\tabcolsep}{4pt}
\small
\begin{tabular}{l cccc}
\toprule
\multicolumn{5}{l}{\textit{(a) AIME-2025}} \\
\midrule
Model & Base & TextGrad & EMO & pass@5 \\
\midrule
Llama-3.2-3B   & 0.0 & \B{6.7} & 3.3 & 3.3 \\
Llama-3.1-8B   & 0.0 & \B{3.3} & 3.3 & 3.3 \\
Qwen2.5-7B     & 6.7 & 13.3 & \B{23.3} & 26.7 \\
Qwen3-8B       & 16.7 & 36.7 & \B{43.3} & \B{60.0} \\
DR1-Qwen3-8B   & 13.3 & \B{20.0} & 16.7 & 33.3 \\
Qwen3.5-9B     & 26.7 & 26.7 & \B{53.3} & \B{80.0} \\
Qwen3-14B      & 16.7 & 30.0 & \B{33.3} & \B{53.3} \\
Qwen3-32B      & 10.0 & \B{43.3} & \B{43.3} & 66.7 \\
\midrule
\multicolumn{5}{l}{\textit{(b) AIME-2026}} \\
\midrule
Model & Base & TextGrad & EMO & pass@5 \\
\midrule
Llama-3.2-3B   & 0.0 & \B{6.7} & 3.3 & 6.7 \\
Llama-3.1-8B   & 0.0 & \B{3.3} & \B{3.3} & 13.3 \\
Qwen2.5-7B     & 3.3 & \B{10.0} & \B{10.0} & 16.7 \\
Qwen3-8B       & 23.3 & \B{50.0} & \B{50.0} & \B{66.7} \\
DR1-Qwen3-8B   & 3.3 & \B{16.7} & \B{16.7} & 43.3 \\
Qwen3.5-9B     & 23.3 & 26.7 & \B{36.7} & \B{73.3} \\
Qwen3-14B      & 23.3 & \B{36.7} & 33.3 & 56.7 \\
Qwen3-32B      & 26.7 & 43.3 & \B{53.3} & \B{73.3} \\
\bottomrule
\end{tabular}
\end{table}

\subsubsection{Accuracy vs.\ Iteration for All Models}
\label{app:iter_curves}


Figures~\ref{fig:app_gsm8k}--\ref{fig:app_aime2026} plot the accuracy of the judge-selected response (solid lines) and the pass@$n$ oracle upper bound (dashed lines) against the EMO iteration $n$, for all 8 models on each of the six datasets.
All curves are computed from the greedy main experiments, with the baseline (grey dashed line) scored under the same unified judge.
Throughout this subsection, iteration $n \in \{1,\dots,5\}$ counts the baseline round as $n{=}1$ (i.e., iteration $n$ corresponds to index iter$n{-}1$ elsewhere).

Three patterns emerge.
\textbf{(i) Early convergence.}
Across the 44 (model, task) pairs with a non-saturated gain ($\geq 2$ points), iteration 2 already realizes on average about half of the final gain of EMO, and on saturated tasks (GSM8K for the Qwen3 family) the curve flattens after a single iteration.
The main exceptions are the hardest tasks (AIME, GPQA-D for weak models), where gains accrue through iteration 4---consistent with headroom being unlocked late when the base policy is far from the target.
\textbf{(ii) Steering benefits depend on the task beyond its answer format.}
On MMLU-Pro the pass@5 of EMO exceeds that of TextGrad for all 8 models (mean gap $+9.6$ points, up to $+17.9$ for Llama-3.1-8B: $74.5$ vs.\ $56.6$), and the selected accuracy likewise wins $8/8$.
On free-form mathematics (MATH-500, GSM8K), the mean pass@5 gap is at most $1.4$ points, suggesting that context refinement already provides most of the observed coverage.
However, GPQA-D also uses multiple-choice answers and exhibits mixed selected-accuracy effects, so answer format alone cannot account for the MMLU-Pro pattern.
The results motivate testing whether the quality of context-induced distribution shifts, rather than output format, predicts when steering helps.
\textbf{(iii) Candidate coverage and selection can improve at different rates.}
On Qwen3-14B/GPQA-D, EMO reaches $79.8$ pass@5 against $65.2$ selected accuracy.
The gap shows that correct candidates are available for additional questions but are not always selected; by itself, it does not measure semantic diversity.
Rising coverage across rounds establishes that later attempts solve previously uncovered questions, motivating stronger verification to realize these gains (Section~\ref{sec:iteration_selection}).

\begin{figure}[h]
\centering
\includegraphics[width=\linewidth]{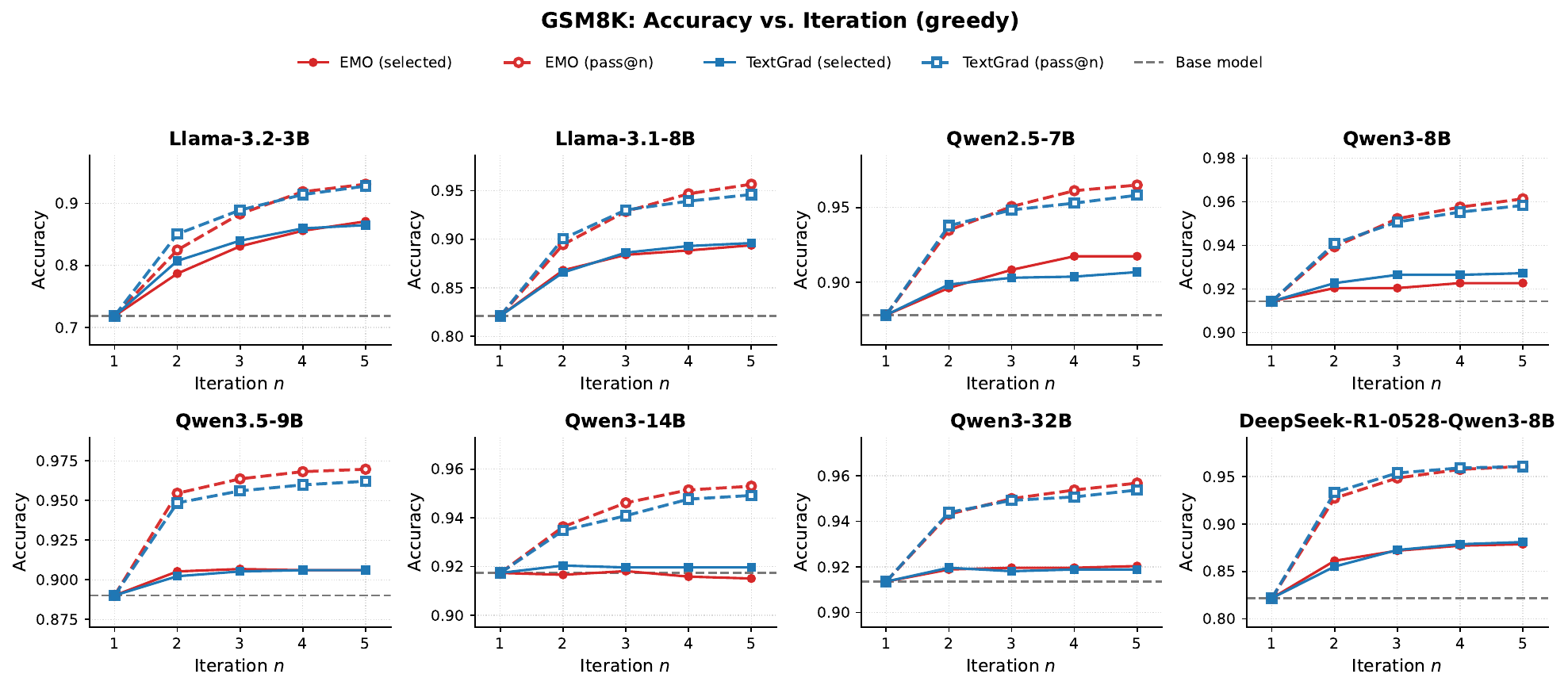}
\caption{GSM8K: accuracy vs.\ EMO iteration for all models (greedy). Solid: judge-selected accuracy; dashed: pass@$n$ (oracle); grey: base model.}
\label{fig:app_gsm8k}
\end{figure}

\begin{figure}[h]
\centering
\includegraphics[width=\linewidth]{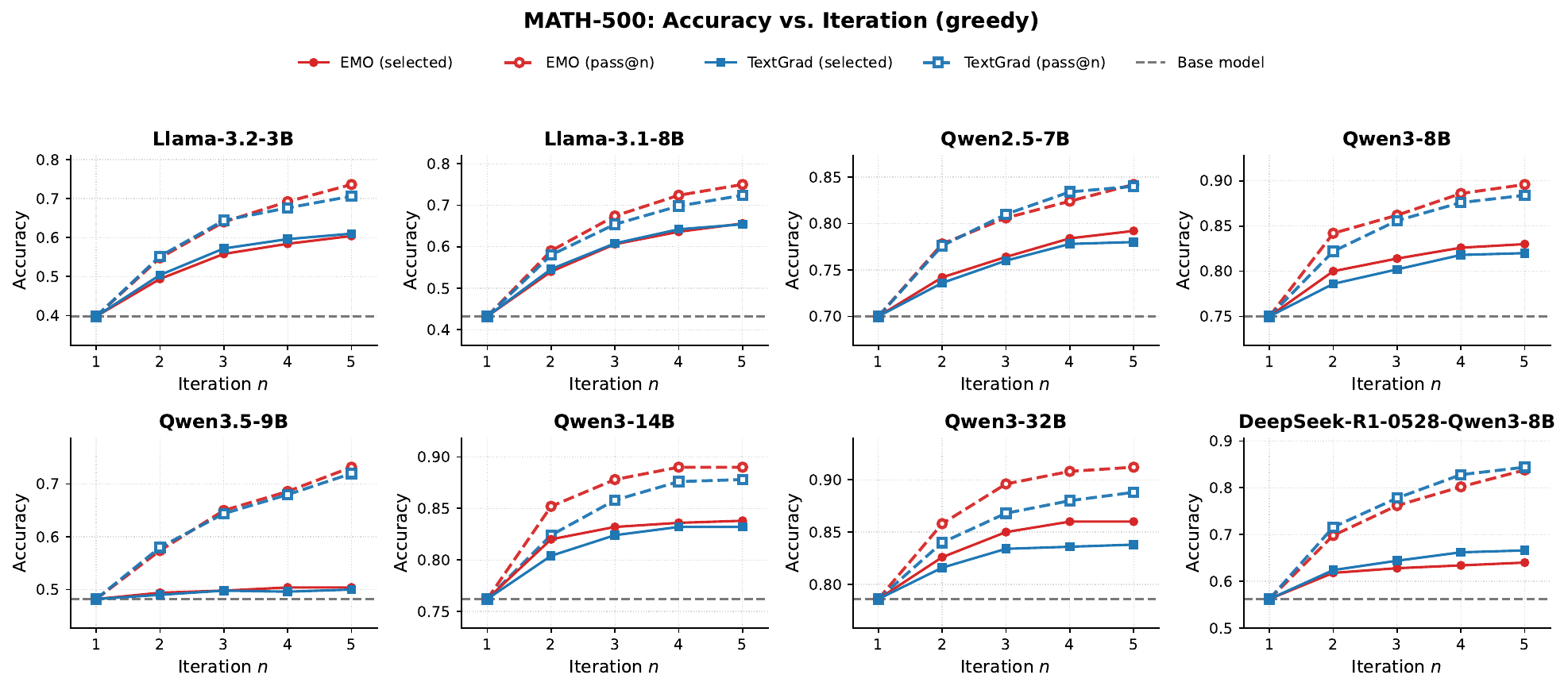}
\caption{MATH-500: accuracy vs.\ EMO iteration for all models (greedy).}
\label{fig:app_math_500}
\end{figure}

\begin{figure}[h]
\centering
\includegraphics[width=\linewidth]{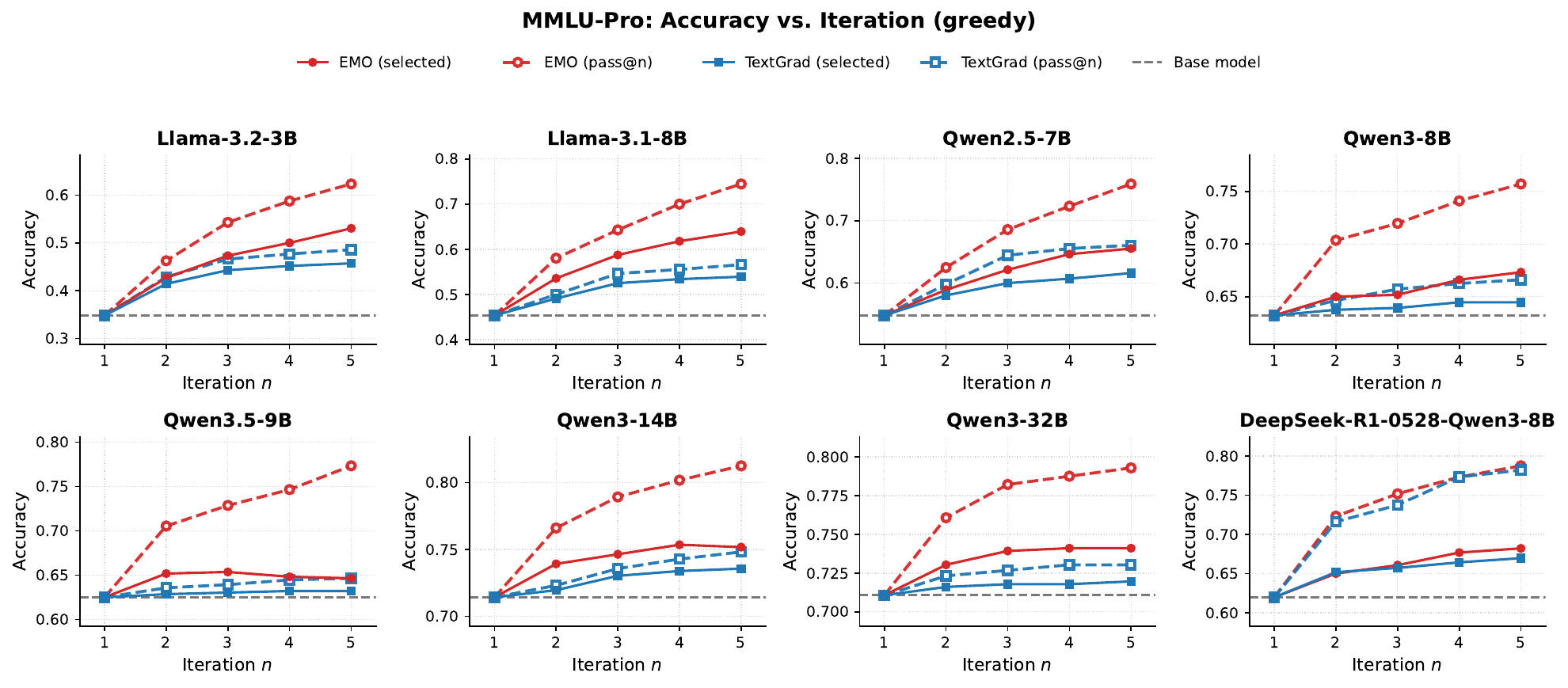}
\caption{MMLU-Pro: accuracy vs.\ EMO iteration for all models (greedy).}
\label{fig:app_mmlu}
\end{figure}

\begin{figure}[h]
\centering
\includegraphics[width=\linewidth]{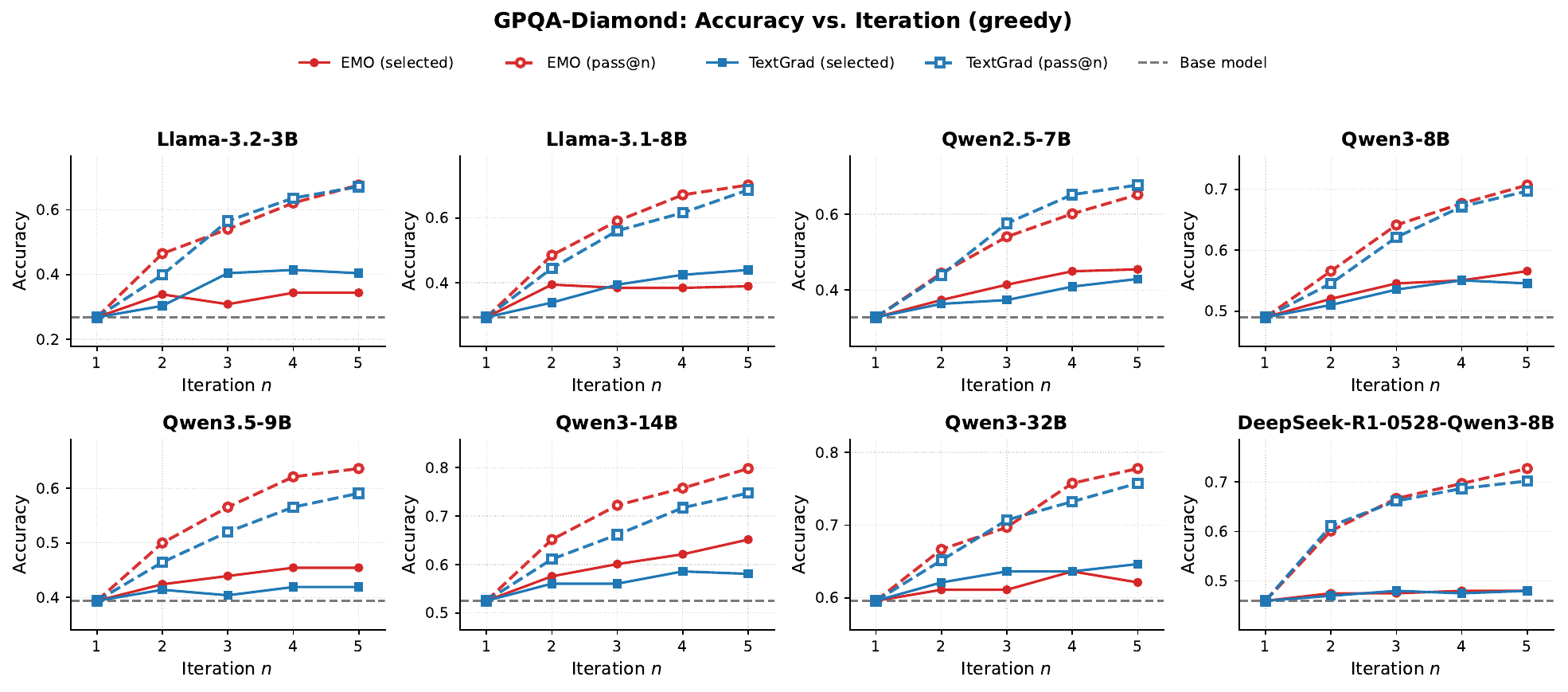}
\caption{GPQA-Diamond: accuracy vs.\ EMO iteration for all models (greedy).}
\label{fig:app_gpqa}
\end{figure}

\begin{figure}[h]
\centering
\includegraphics[width=\linewidth]{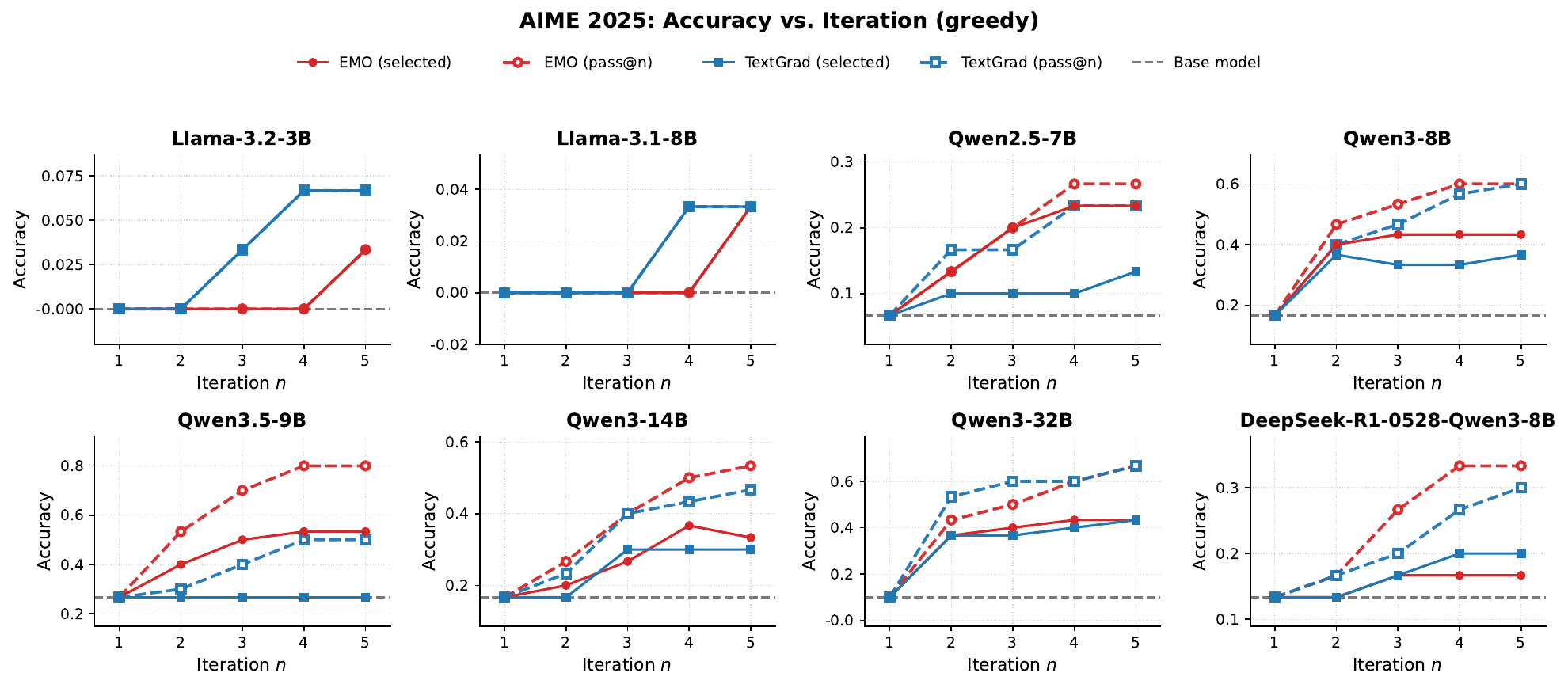}
\caption{AIME-2025: accuracy vs.\ EMO iteration for all models (greedy).}
\label{fig:app_aime2025}
\end{figure}

\begin{figure}[h]
\centering
\includegraphics[width=\linewidth]{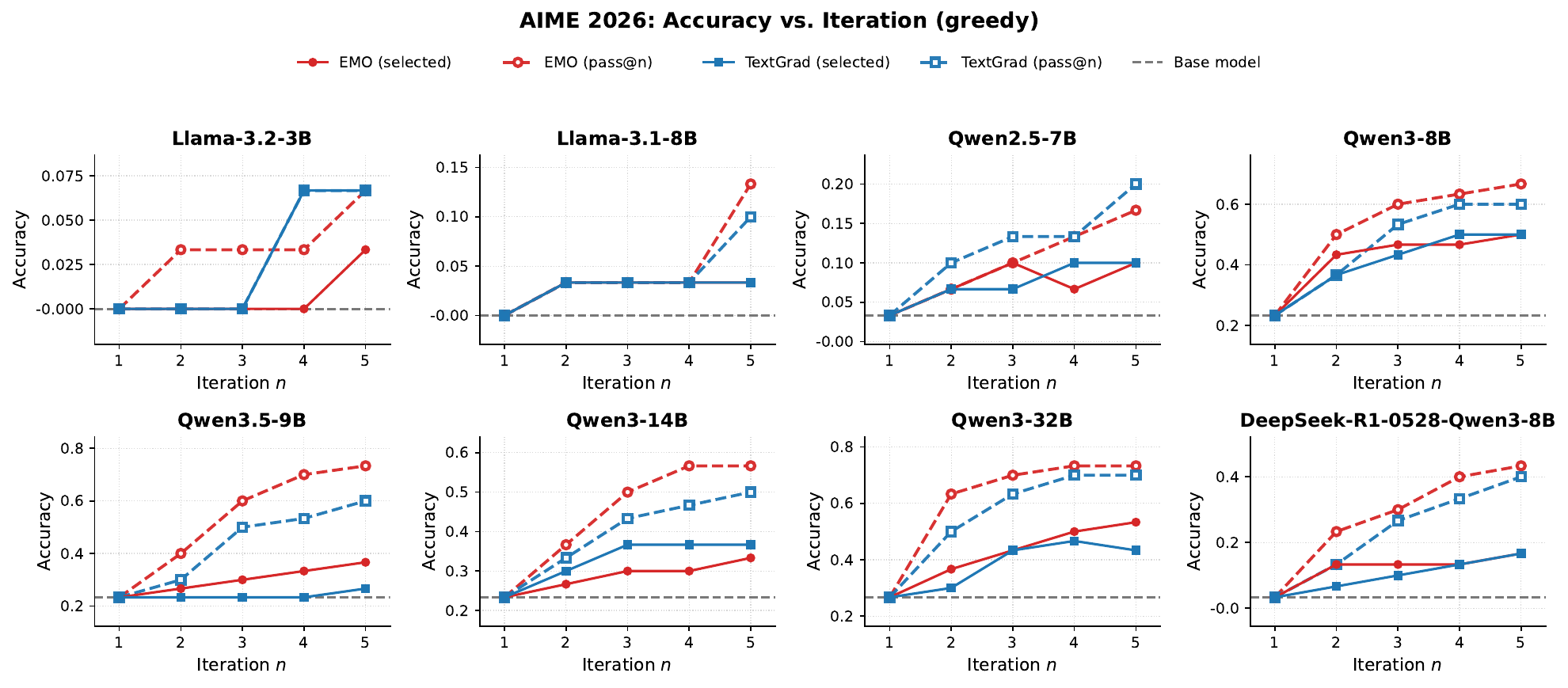}
\caption{AIME-2026: accuracy vs.\ EMO iteration for all models (greedy).}
\label{fig:app_aime2026}
\end{figure}
\subsubsection{Component Ablation on Qwen3-32B}
\label{app:component_32b}


Table~\ref{tab:component_32b} repeats the component ablation of Table~\ref{tab:component} on Qwen3-32B.
The cross-scale comparison distinguishes three forms of component dependence.
First, steering and external judging can contribute differently even within AIME.
With the external judge, the M-step ties context-only optimization on AIME-2025 ($43.3$) but adds $10.0$ points on AIME-2026 ($53.3$ vs.\ $43.3$).
With steering retained, external judging adds $13.3$ and $6.6$ points over self-judging, respectively; neither component is uniformly dominant across editions.
Second, steering on GPQA-D loses $2.5$ points under both judge choices ($62.1$ vs.\ $64.6$ externally judged; $61.6$ vs.\ $64.1$ self-judged).
This differs from Qwen3-8B, where replacing the self-judge with the external judge changes the steering increment from $-3.5$ to $+2.1$.
Feedback quality and steering utility can interact, but external judging does not universally resolve an unfavorable steering effect.
Third, with steering retained, the external-judge increment on MMLU-Pro changes from $+2.3$ at 8B to $-1.1$ at 32B, where self-judging reaches $75.4$ versus $74.3$.
This reversal motivates model- and task-dependent judge allocation, but does not establish better intrinsic self-evaluation at 32B: judge replacement changes both feedback and candidate selection, and the greedy comparison does not estimate uncertainty across runs.
Occasional repetitive-generation failures under self-judging are documented in Appendix~\ref{app:settings}.

\begin{table}[h]
\centering
\caption{\textbf{Component ablation on Qwen3-32B} (accuracy \%, greedy), mirroring Table~\ref{tab:component} (Qwen3-8B). AIME rows are reported per year (30 problems each).}
\label{tab:component_32b}
\setlength{\tabcolsep}{4pt}
\small
\begin{tabular}{l ccccc}
\toprule
Dataset & Base & EMO & \shortstack{w/o\\M-step} & \shortstack{w/o\\Judge} & \shortstack{w/o\\Both} \\
\midrule
AIME-2025 & 10.0 & \B{43.3} & \B{43.3} & 30.0 & 26.7 \\
AIME-2026 & 26.7 & \B{53.3} & 43.3 & 46.7 & 43.3 \\
GPQA-D    & 59.6 & 62.1 & \B{64.6} & 61.6 & 64.1 \\
GSM8K     & 91.4 & \B{92.0} & 91.9 & 91.7 & 91.7 \\
MATH-500  & 78.6 & \B{86.0} & 83.8 & 83.2 & 82.6 \\
MMLU-Pro  & 71.4 & 74.3 & 72.3 & \B{75.4} & 74.8 \\
\bottomrule
\end{tabular}
\end{table}

\subsubsection{Full Compute-Matched Results}
\label{app:compute_full}


Tables~\ref{tab:compute_emopt}--\ref{tab:compute_pass} report the full sampling results for three models and three tasks (mean$\pm$std over three seeds, temperature $0.6$, top-$p=0.95$).
EMO uses $N=2k+1$ target-model generation equivalents; SC and BoN are reported at every budget $N=1,\ldots,10$.
The candidate pools contain $k+1$ responses for EMO and $N$ for SC/BoN.

\paragraph{Selected accuracy.}
At $k=4$ ($N=9$), EMO exceeds BoN in mean accuracy on eight of nine pairs; Qwen3-8B/GPQA-D is the exception ($60.3$ vs.\ $60.4$).
The Qwen3-8B/MATH-500 means are also close ($85.0$ vs.\ $84.9$), and SC is higher on this pair ($86.1$).
The clearest EMO advantages over BoN are on the six Llama tasks ($+5.2$, $+4.8$, and $+5.7$ points on Llama-3.1-8B GPQA-D, MATH-500, and MMLU-Pro; $+6.5$, $+7.6$, and $+12.7$ on Llama-3.2-3B) and Qwen3-8B/MMLU-Pro ($+2.5$).
SC and BoN trajectories need not be monotone: for example, Qwen3-8B/GPQA-D BoN accuracy decreases from $60.6$ at $N=8$ to $59.8$ at $N=10$.
Likewise, EMO's mean selected accuracy increases with each round in eight of nine pairs; the exception is Llama-3.1-8B/GPQA-D, which decreases from $37.5$ at $k=3$ to $37.4$ at $k=4$.

\paragraph{Oracle coverage.}
At $N=9$, EMO's five-candidate coverage exceeds the nine-sample SC/BoN coverage in seven of nine pairs.
The largest gap is on Llama-3.1-8B/GPQA-D ($73.7$ vs.\ $60.3$); the two exceptions are Qwen3-8B/GPQA-D ($79.1$ vs.\ $80.0$) and Qwen3-8B/MATH-500 ($90.4$ vs.\ $90.9$), where independent sampling retains an edge.
Coverage can continue to improve while selected accuracy changes little, as on Llama-3.1-8B/GPQA-D, indicating that candidate availability and judge selection place different limits on the final result.

\paragraph{How coverage translates into selected accuracy.}
Because a correct selected response must belong to a pool containing a correct response, selected accuracy divided by oracle coverage measures the fraction of covered questions on which selection succeeds.
We compute this descriptive ratio from the reported three-seed means; with equal evaluation-set sizes, it corresponds to pooling question--seed outcomes, rather than averaging per-seed ratios.
At $N=9$ on Llama-3.1-8B/GPQA-D, the ratio is $37.4/73.7 \approx 50.7\%$ for EMO and $32.2/60.3 \approx 53.4\%$ for BoN.
EMO's accuracy advantage therefore coexists with a lower fraction of available successes realized by selection: its greater coverage is essential to the observed advantage.
On Llama-3.1-8B/MATH-500, both coverage ($73.4$ vs.\ $71.9$) and this fraction ($89.4\%$ vs.\ $84.6\%$) favor EMO.
The same pattern appears on Llama-3.2-3B/MATH-500 (Figure~\ref{fig:compute-b}): EMO's coverage is $70.6$ versus BoN's $66.5$, and its selected-to-covered fraction is $60.3/70.6 \approx 85.4\%$ versus $52.7/66.5 \approx 79.2\%$.
These are descriptive relationships, not isolated causal effects of ranking: candidate pools and the questions they cover differ across methods.
They show why a larger absolute oracle--accuracy gap need not indicate less useful generation, and why selection quality alone cannot explain all accuracy gains.

\paragraph{Repeated sampling and feedback address different limitations.}
At $N=9$ on Llama-3.1-8B/MMLU-Pro, SC and BoN share $76.4$ oracle coverage but achieve $51.4$ and $62.5$ accuracy, respectively.
Thus many available correct responses are not recovered by majority voting, while judge selection realizes more of the same pool's potential.
EMO further raises coverage to $83.2$ and accuracy to $68.2$ with only five candidates, showing additional value from changing how subsequent attempts are generated.
For Qwen3-8B/MATH-500, by contrast, SC is strongest ($86.1$), and EMO has no coverage advantage over independent sampling ($90.4$ vs.\ $90.9$).
This contrast suggests that feedback-guided adaptation is most useful where independent sampling and its selection rule leave recoverable errors; it need not dominate when repeated sampling already supplies effective candidates and reliable consensus.

\begin{table}[h]
\centering
\caption{\textbf{EMO vs.\ iteration under top-$p$ sampling} (accuracy \%, mean$\pm$std over 3 seeds). iter $k$ uses $N{=}2k{+}1$ target-model generation equivalents; \emph{Pass} is oracle coverage over the $k+1$ candidate responses.}
\label{tab:compute_emopt}
\setlength{\tabcolsep}{4.5pt}
\small
\begin{tabular}{ll ccc cc}
\toprule
Model & Metric & \multicolumn{5}{c}{iter ($N$)} \\
\cmidrule(lr){3-7}
 & & 0 (1) & 1 (3) & 2 (5) & 3 (7) & 4 (9) \\
\midrule
\multirow{6}{*}{Qwen3-8B}
 & acc \quad GPQA-D    & 54.4{\tiny$\pm$1.4} & 55.4{\tiny$\pm$0.2} & 57.1{\tiny$\pm$1.1} & 58.4{\tiny$\pm$1.3} & \B{60.3{\tiny$\pm$2.1}} \\
 & pass               & 54.4{\tiny$\pm$1.4} & 64.0{\tiny$\pm$0.2} & 70.0{\tiny$\pm$1.7} & 75.6{\tiny$\pm$1.6} & \B{79.1{\tiny$\pm$0.6}} \\
\addlinespace[2pt]
 & acc \quad MATH-500  & 79.6{\tiny$\pm$0.4} & 83.5{\tiny$\pm$0.6} & 84.3{\tiny$\pm$0.7} & 84.9{\tiny$\pm$0.6} & \B{85.0{\tiny$\pm$0.6}} \\
 & pass               & 79.6{\tiny$\pm$0.4} & 86.9{\tiny$\pm$0.6} & 88.9{\tiny$\pm$0.2} & 90.0{\tiny$\pm$0.2} & \B{90.4{\tiny$\pm$0.2}} \\
\addlinespace[2pt]
 & acc \quad MMLU-Pro  & 63.9{\tiny$\pm$1.3} & 67.9{\tiny$\pm$1.1} & 69.4{\tiny$\pm$0.6} & 70.0{\tiny$\pm$0.8} & \B{70.9{\tiny$\pm$0.7}} \\
 & pass               & 63.9{\tiny$\pm$1.3} & 75.8{\tiny$\pm$1.5} & 80.3{\tiny$\pm$1.7} & 83.0{\tiny$\pm$2.2} & \B{84.6{\tiny$\pm$1.7}} \\
\midrule
\multirow{6}{*}{Llama-3.1-8B}
 & acc \quad GPQA-D    & 29.6{\tiny$\pm$1.3} & 35.5{\tiny$\pm$3.6} & 36.5{\tiny$\pm$2.5} & \B{37.5{\tiny$\pm$2.4}} & 37.4{\tiny$\pm$0.7} \\
 & pass               & 29.6{\tiny$\pm$1.3} & 49.3{\tiny$\pm$2.7} & 62.1{\tiny$\pm$1.2} & 70.2{\tiny$\pm$1.2} & \B{73.7{\tiny$\pm$0.4}} \\
\addlinespace[2pt]
 & acc \quad MATH-500  & 45.9{\tiny$\pm$1.6} & 55.5{\tiny$\pm$1.8} & 59.7{\tiny$\pm$1.6} & 63.7{\tiny$\pm$1.5} & \B{65.6{\tiny$\pm$1.6}} \\
 & pass               & 45.9{\tiny$\pm$1.6} & 59.3{\tiny$\pm$1.7} & 66.1{\tiny$\pm$1.7} & 70.5{\tiny$\pm$2.1} & \B{73.4{\tiny$\pm$2.0}} \\
\addlinespace[2pt]
 & acc \quad MMLU-Pro  & 45.1{\tiny$\pm$1.0} & 59.4{\tiny$\pm$2.6} & 63.7{\tiny$\pm$1.5} & 65.7{\tiny$\pm$1.1} & \B{68.2{\tiny$\pm$0.8}} \\
 & pass               & 45.1{\tiny$\pm$1.0} & 66.9{\tiny$\pm$2.1} & 76.2{\tiny$\pm$2.1} & 80.2{\tiny$\pm$1.3} & \B{83.2{\tiny$\pm$1.1}} \\
\midrule
\multirow{6}{*}{Llama-3.2-3B}
 & acc \quad GPQA-D    & 24.4{\tiny$\pm$1.7} & 28.3{\tiny$\pm$2.3} & 32.3{\tiny$\pm$0.8} & 34.0{\tiny$\pm$1.7} & \B{36.5{\tiny$\pm$2.3}} \\
 & pass               & 24.4{\tiny$\pm$1.7} & 43.9{\tiny$\pm$1.5} & 56.6{\tiny$\pm$1.5} & 66.2{\tiny$\pm$2.7} & \B{71.2{\tiny$\pm$1.9}} \\
\addlinespace[2pt]
 & acc \quad MATH-500  & 40.3{\tiny$\pm$0.8} & 50.7{\tiny$\pm$0.9} & 55.5{\tiny$\pm$1.0} & 58.6{\tiny$\pm$1.0} & \B{60.3{\tiny$\pm$1.3}} \\
 & pass               & 40.3{\tiny$\pm$0.8} & 55.8{\tiny$\pm$0.7} & 63.2{\tiny$\pm$0.9} & 67.4{\tiny$\pm$0.9} & \B{70.6{\tiny$\pm$0.7}} \\
\addlinespace[2pt]
 & acc \quad MMLU-Pro  & 35.5{\tiny$\pm$1.3} & 53.2{\tiny$\pm$1.5} & 60.4{\tiny$\pm$1.5} & 64.1{\tiny$\pm$0.5} & \B{66.2{\tiny$\pm$0.5}} \\
 & pass               & 35.5{\tiny$\pm$1.3} & 59.9{\tiny$\pm$1.5} & 71.5{\tiny$\pm$0.7} & 77.9{\tiny$\pm$0.6} & \B{81.6{\tiny$\pm$0.4}} \\
\bottomrule
\end{tabular}
\end{table}

\begin{table}[h]
\centering
\caption{\textbf{Self-consistency (SC@$N$) under top-$p$ sampling} (accuracy \%, mean$\pm$std over 3 seeds), majority vote over $N$ independent samples.}
\label{tab:compute_sc}
\setlength{\tabcolsep}{2pt}
\scriptsize
\begin{tabular}{l l ccccccccccc}
\toprule
Model & Task & \multicolumn{10}{c}{SC@$N$, $N{=}1,\dots,10$} \\
\cmidrule(lr){3-12}
 & & 1 & 2 & 3 & 4 & 5 & 6 & 7 & 8 & 9 & 10 \\
\midrule
\multirow{3}{*}{Qwen3-8B}
 & GPQA-D   & 52.4{\tiny$\pm$1.4} & 57.4{\tiny$\pm$1.0} & 56.4{\tiny$\pm$1.9} & 58.1{\tiny$\pm$0.8} & 57.9{\tiny$\pm$1.7} & \B{58.9{\tiny$\pm$2.3}} & 58.4{\tiny$\pm$1.4} & 58.4{\tiny$\pm$3.2} & 57.9{\tiny$\pm$2.8} & 57.6{\tiny$\pm$1.1} \\
 & MATH-500 & 79.2{\tiny$\pm$0.2} & 82.7{\tiny$\pm$0.7} & 84.2{\tiny$\pm$1.1} & 84.9{\tiny$\pm$0.7} & 85.1{\tiny$\pm$0.2} & 85.6{\tiny$\pm$0.0} & 85.8{\tiny$\pm$0.4} & 86.0{\tiny$\pm$0.4} & 86.1{\tiny$\pm$0.2} & \B{86.3{\tiny$\pm$0.5}} \\
 & MMLU-Pro & 65.5{\tiny$\pm$1.2} & 66.4{\tiny$\pm$0.7} & 66.8{\tiny$\pm$1.2} & 67.2{\tiny$\pm$0.9} & 67.6{\tiny$\pm$0.8} & 67.6{\tiny$\pm$0.8} & \B{67.9{\tiny$\pm$0.3}} & 67.7{\tiny$\pm$1.0} & 67.1{\tiny$\pm$1.2} & 67.7{\tiny$\pm$0.6} \\
\midrule
\multirow{3}{*}{Llama-3.1-8B}
 & GPQA-D   & 26.4{\tiny$\pm$0.9} & 27.8{\tiny$\pm$1.5} & 28.6{\tiny$\pm$1.7} & 27.8{\tiny$\pm$1.2} & 27.6{\tiny$\pm$2.1} & 29.0{\tiny$\pm$1.7} & \B{30.0{\tiny$\pm$1.3}} & 28.8{\tiny$\pm$1.1} & 28.5{\tiny$\pm$0.9} & 28.8{\tiny$\pm$1.5} \\
 & MATH-500 & 45.3{\tiny$\pm$1.9} & 51.7{\tiny$\pm$1.8} & 53.4{\tiny$\pm$1.9} & 54.8{\tiny$\pm$1.9} & 55.5{\tiny$\pm$1.3} & 55.7{\tiny$\pm$1.0} & 56.1{\tiny$\pm$0.4} & 56.0{\tiny$\pm$0.3} & 56.9{\tiny$\pm$0.2} & \B{57.5{\tiny$\pm$0.9}} \\
 & MMLU-Pro & 45.1{\tiny$\pm$0.6} & \B{51.7{\tiny$\pm$0.9}} & 50.4{\tiny$\pm$1.5} & 50.8{\tiny$\pm$0.3} & 49.8{\tiny$\pm$0.2} & 51.7{\tiny$\pm$0.6} & 51.0{\tiny$\pm$1.1} & 51.1{\tiny$\pm$0.7} & 51.4{\tiny$\pm$0.8} & 51.4{\tiny$\pm$0.6} \\
\midrule
\multirow{3}{*}{Llama-3.2-3B}
 & GPQA-D   & 25.8{\tiny$\pm$2.2} & \B{27.8{\tiny$\pm$1.6}} & 26.4{\tiny$\pm$2.4} & 27.1{\tiny$\pm$1.9} & 27.1{\tiny$\pm$2.3} & 26.9{\tiny$\pm$1.3} & 26.4{\tiny$\pm$0.2} & 26.4{\tiny$\pm$1.0} & 26.4{\tiny$\pm$0.9} & 25.9{\tiny$\pm$1.0} \\
 & MATH-500   & 37.8{\tiny$\pm$1.3} & 43.3{\tiny$\pm$1.8} & 45.3{\tiny$\pm$2.7} & 45.9{\tiny$\pm$1.3} & 47.7{\tiny$\pm$1.9} & 48.5{\tiny$\pm$1.9} & 48.7{\tiny$\pm$1.9} & 48.5{\tiny$\pm$2.1} & 48.6{\tiny$\pm$2.1} & \B{49.3{\tiny$\pm$1.6}} \\
 & MMLU-Pro   & 36.5{\tiny$\pm$1.7} & 42.0{\tiny$\pm$1.3} & 40.5{\tiny$\pm$0.9} & 41.7{\tiny$\pm$1.1} & 41.2{\tiny$\pm$0.8} & 41.4{\tiny$\pm$0.5} & 40.9{\tiny$\pm$0.8} & 41.0{\tiny$\pm$0.7} & 42.0{\tiny$\pm$0.5} & \B{42.1{\tiny$\pm$0.9}} \\
\bottomrule
\end{tabular}
\end{table}

\begin{table}[h]
\centering
\caption{\textbf{Best-of-$N$ (BoN@$N$) under top-$p$ sampling} (accuracy \%, mean$\pm$std over 3 seeds), judge-selected best of $N$ independent samples.}
\label{tab:compute_bon}
\setlength{\tabcolsep}{2pt}
\scriptsize
\begin{tabular}{l l ccccccccccc}
\toprule
Model & Task & \multicolumn{10}{c}{BoN@$N$, $N{=}1,\dots,10$} \\
\cmidrule(lr){3-12}
 & & 1 & 2 & 3 & 4 & 5 & 6 & 7 & 8 & 9 & 10 \\
\midrule
\multirow{3}{*}{Qwen3-8B}
 & GPQA-D   & 52.4{\tiny$\pm$1.4} & 57.4{\tiny$\pm$1.0} & 58.2{\tiny$\pm$0.6} & 59.9{\tiny$\pm$0.6} & 59.8{\tiny$\pm$0.6} & 59.9{\tiny$\pm$0.5} & 60.3{\tiny$\pm$0.2} & \B{60.6{\tiny$\pm$0.0}} & 60.4{\tiny$\pm$0.9} & 59.8{\tiny$\pm$0.6} \\
 & MATH-500 & 79.2{\tiny$\pm$0.2} & 82.7{\tiny$\pm$0.7} & 84.1{\tiny$\pm$0.7} & 84.6{\tiny$\pm$0.6} & 84.7{\tiny$\pm$0.5} & 84.9{\tiny$\pm$0.7} & 85.0{\tiny$\pm$0.6} & 85.0{\tiny$\pm$0.7} & 84.9{\tiny$\pm$0.6} & \B{85.1{\tiny$\pm$0.6}} \\
 & MMLU-Pro & 65.5{\tiny$\pm$1.2} & 66.4{\tiny$\pm$0.7} & 67.1{\tiny$\pm$1.0} & 67.7{\tiny$\pm$0.4} & 68.2{\tiny$\pm$0.7} & 68.0{\tiny$\pm$0.5} & 68.2{\tiny$\pm$0.6} & 68.3{\tiny$\pm$0.9} & 68.4{\tiny$\pm$0.5} & \B{68.7{\tiny$\pm$0.4}} \\
\midrule
\multirow{3}{*}{Llama-3.1-8B}
 & GPQA-D   & 26.4{\tiny$\pm$0.9} & 27.8{\tiny$\pm$1.5} & 28.6{\tiny$\pm$1.3} & 29.5{\tiny$\pm$1.0} & 29.5{\tiny$\pm$0.9} & 30.0{\tiny$\pm$0.2} & 30.8{\tiny$\pm$0.4} & 31.6{\tiny$\pm$0.9} & 32.2{\tiny$\pm$0.9} & \B{33.0{\tiny$\pm$0.5}} \\
 & MATH-500 & 45.3{\tiny$\pm$1.9} & 51.7{\tiny$\pm$1.8} & 54.5{\tiny$\pm$1.5} & 55.6{\tiny$\pm$1.3} & 56.9{\tiny$\pm$1.2} & 58.1{\tiny$\pm$0.9} & 59.1{\tiny$\pm$0.1} & 60.1{\tiny$\pm$0.2} & 60.8{\tiny$\pm$0.0} & \B{61.3{\tiny$\pm$0.2}} \\
 & MMLU-Pro & 45.1{\tiny$\pm$0.6} & 51.7{\tiny$\pm$0.9} & 55.1{\tiny$\pm$1.0} & 56.9{\tiny$\pm$1.2} & 58.0{\tiny$\pm$1.1} & 59.6{\tiny$\pm$0.5} & 61.0{\tiny$\pm$0.8} & 61.9{\tiny$\pm$0.6} & 62.5{\tiny$\pm$0.5} & \B{62.9{\tiny$\pm$0.8}} \\
\midrule
\multirow{3}{*}{Llama-3.2-3B}
 & GPQA-D   & 25.8{\tiny$\pm$2.2} & 27.8{\tiny$\pm$1.6} & 28.1{\tiny$\pm$2.0} & 28.5{\tiny$\pm$2.3} & 29.1{\tiny$\pm$1.7} & 29.3{\tiny$\pm$1.1} & 29.5{\tiny$\pm$1.4} & 29.8{\tiny$\pm$1.6} & 30.0{\tiny$\pm$1.7} & \B{30.3{\tiny$\pm$1.8}} \\
 & MATH-500   & 37.8{\tiny$\pm$1.3} & 43.3{\tiny$\pm$1.8} & 45.8{\tiny$\pm$1.4} & 47.8{\tiny$\pm$1.4} & 49.6{\tiny$\pm$1.0} & 50.9{\tiny$\pm$1.3} & 51.9{\tiny$\pm$1.5} & 52.4{\tiny$\pm$1.3} & 52.7{\tiny$\pm$1.5} & \B{53.7{\tiny$\pm$1.1}} \\
 & MMLU-Pro   & 36.5{\tiny$\pm$1.7} & 42.0{\tiny$\pm$1.3} & 46.4{\tiny$\pm$0.8} & 48.2{\tiny$\pm$1.3} & 49.8{\tiny$\pm$1.1} & 50.8{\tiny$\pm$0.9} & 51.9{\tiny$\pm$0.7} & 52.8{\tiny$\pm$0.4} & 53.5{\tiny$\pm$0.4} & \B{54.2{\tiny$\pm$0.4}} \\
\bottomrule
\end{tabular}
\end{table}

\begin{table}[h]
\centering
\caption{\textbf{Oracle coverage at matched generation budgets} (accuracy \%, mean$\pm$std over 3 seeds): at least one candidate is correct. At budget $N=2k+1$, SC/BoN use $N$ independent samples and EMO uses $k+1$ candidate responses.}
\label{tab:compute_pass}
\setlength{\tabcolsep}{4pt}
\small
\begin{tabular}{ll cc cc cc}
\toprule
\multirow{2}{*}{Model} & \multirow{2}{*}{Task} & \multicolumn{2}{c}{$N{=}3$} & \multicolumn{2}{c}{$N{=}5$} & \multicolumn{2}{c}{$N{=}9$} \\
\cmidrule(lr){3-4}\cmidrule(lr){5-6}\cmidrule(lr){7-8}
 & & SC/BoN & EMO & SC/BoN & EMO & SC/BoN & EMO \\
\midrule
\multirow{3}{*}{Qwen3-8B}
 & GPQA-D   & \B{69.4{\tiny$\pm$0.6}} & 64.0{\tiny$\pm$0.2} & \B{74.7{\tiny$\pm$1.1}} & 70.0{\tiny$\pm$1.7} & \B{80.0{\tiny$\pm$0.2}} & 79.1{\tiny$\pm$0.6} \\
 & MATH-500 & \B{87.8{\tiny$\pm$0.0}} & 86.9{\tiny$\pm$0.6} & \B{89.7{\tiny$\pm$0.2}} & 88.9{\tiny$\pm$0.2} & \B{90.9{\tiny$\pm$0.2}} & 90.4{\tiny$\pm$0.2} \\
 & MMLU-Pro & 75.2{\tiny$\pm$0.8} & \B{75.8{\tiny$\pm$1.5}} & 78.7{\tiny$\pm$0.3} & \B{80.3{\tiny$\pm$1.7}} & 81.9{\tiny$\pm$0.5} & \B{84.6{\tiny$\pm$1.7}} \\
\midrule
\multirow{3}{*}{Llama-3.1-8B}
 & GPQA-D   & 45.1{\tiny$\pm$1.3} & \B{49.3{\tiny$\pm$2.7}} & 51.2{\tiny$\pm$0.9} & \B{62.1{\tiny$\pm$1.2}} & 60.3{\tiny$\pm$2.7} & \B{73.7{\tiny$\pm$0.4}} \\
 & MATH-500 & \B{60.7{\tiny$\pm$1.6}} & 59.3{\tiny$\pm$1.7} & 65.9{\tiny$\pm$1.1} & \B{66.1{\tiny$\pm$1.7}} & 71.9{\tiny$\pm$0.4} & \B{73.4{\tiny$\pm$2.0}} \\
 & MMLU-Pro & 62.3{\tiny$\pm$1.1} & \B{66.9{\tiny$\pm$2.1}} & 68.9{\tiny$\pm$0.4} & \B{76.2{\tiny$\pm$2.1}} & 76.4{\tiny$\pm$0.2} & \B{83.2{\tiny$\pm$1.1}} \\
\midrule
\multirow{3}{*}{Llama-3.2-3B}
 & GPQA-D   & \B{45.3{\tiny$\pm$2.5}} & 43.9{\tiny$\pm$1.5} & 53.2{\tiny$\pm$1.7} & \B{56.6{\tiny$\pm$1.5}} & 63.0{\tiny$\pm$0.9} & \B{71.2{\tiny$\pm$1.9}} \\
 & MATH-500 & 52.9{\tiny$\pm$1.9} & \B{55.8{\tiny$\pm$0.7}} & 59.1{\tiny$\pm$2.3} & \B{63.2{\tiny$\pm$0.9}} & 66.5{\tiny$\pm$1.5} & \B{70.6{\tiny$\pm$0.7}} \\
 & MMLU-Pro & 54.0{\tiny$\pm$0.7} & \B{59.9{\tiny$\pm$1.5}} & 60.8{\tiny$\pm$1.3} & \B{71.5{\tiny$\pm$0.7}} & 68.5{\tiny$\pm$1.3} & \B{81.6{\tiny$\pm$0.4}} \\
\bottomrule
\end{tabular}
\end{table}

\subsubsection{Statistical Significance of Sampling Results}
\label{app:significance}

Table~\ref{tab:significance} reports paired significance tests for the top-$p$ sampling experiments (three models $\times$ three tasks).
Each comparison pairs the same question under the same seed, pooling the three seeds ($n$ = questions $\times$ 3).
We apply two-sided paired $t$-tests and Wilcoxon signed-rank tests, with Holm--Bonferroni correction over the 18 reported comparisons per test (nine against Base and nine against BoN@10).
The BoN significance comparisons use ten independent samples against EMO's nine target-model generation equivalents; the matched-budget comparisons in Section~\ref{sec:compute} use BoN@9.
Reported $\Delta$ is the accuracy difference in percentage points between the two arms.

\begin{table}[h]
\centering
\caption{\textbf{Paired significance tests} under top-$p$ sampling ($\Delta$ in points; $n$ = question--seed pairs pooled over seeds $\{0,1,2\}$). $p$-values are Holm-corrected over 18 comparisons; stars mark $^*$ $p{<}0.05$, $^{**}$ $p{<}0.01$, $^{***}$ $p{<}0.001$ (paired $t$-test). Wilcoxon significance decisions agree for all 18 comparisons.}
\label{tab:significance}
\setlength{\tabcolsep}{2.6pt}
\small
\begin{tabular}{ll crr cc}
\toprule
Model & Task & Comparison & $\Delta$ & $n$ & $p_{\mathrm{Holm}}$ & sig. \\
\midrule
\multicolumn{7}{l}{\textit{EMO vs.\ Base}} \\
Qwen3-8B     & GPQA-D   & EMO$-$Base & $+5.9$ & 594  & $2.2{\times}10^{-3}$ & ** \\
Qwen3-8B     & MATH-500 & EMO$-$Base & $+5.4$ & 1500 & $1.0{\times}10^{-16}$ & *** \\
Qwen3-8B     & MMLU-Pro & EMO$-$Base & $+7.0$ & 1680 & $7.7{\times}10^{-18}$ & *** \\
Llama-3.1-8B & GPQA-D   & EMO$-$Base & $+7.7$ & 594  & $2.2{\times}10^{-3}$ & ** \\
Llama-3.1-8B & MATH-500 & EMO$-$Base & $+19.7$ & 1500 & $5.9{\times}10^{-62}$ & *** \\
Llama-3.1-8B & MMLU-Pro & EMO$-$Base & $+23.1$ & 1680 & $3.7{\times}10^{-73}$ & *** \\
Llama-3.2-3B & GPQA-D   & EMO$-$Base & $+12.1$ & 594  & $7.2{\times}10^{-8}$ & *** \\
Llama-3.2-3B & MATH-500 & EMO$-$Base & $+19.9$ & 1500 & $3.6{\times}10^{-62}$ & *** \\
Llama-3.2-3B & MMLU-Pro & EMO$-$Base & $+30.7$ & 1680 & $4.1{\times}10^{-108}$ & *** \\
\midrule
\multicolumn{7}{l}{\textit{EMO vs.\ BoN@10 (same judge selection)}} \\
Qwen3-8B     & GPQA-D   & EMO$-$BoN & $+0.5$ & 594  & $1.00$ & n.s. \\
Qwen3-8B     & MATH-500 & EMO$-$BoN & $-0.1$ & 1500 & $1.00$ & n.s. \\
Qwen3-8B     & MMLU-Pro & EMO$-$BoN & $+2.2$ & 1680 & $9.6{\times}10^{-2}$ & n.s. \\
Llama-3.1-8B & GPQA-D   & EMO$-$BoN & $+4.4$ & 594  & $1.3{\times}10^{-1}$ & n.s. \\
Llama-3.1-8B & MATH-500 & EMO$-$BoN & $+4.3$ & 1500 & $1.1{\times}10^{-3}$ & ** \\
Llama-3.1-8B & MMLU-Pro & EMO$-$BoN & $+5.2$ & 1680 & $4.8{\times}10^{-5}$ & *** \\
Llama-3.2-3B & GPQA-D   & EMO$-$BoN & $+6.2$ & 594  & $2.5{\times}10^{-2}$ & * \\
Llama-3.2-3B & MATH-500 & EMO$-$BoN & $+6.6$ & 1500 & $4.2{\times}10^{-7}$ & *** \\
Llama-3.2-3B & MMLU-Pro & EMO$-$BoN & $+12.1$ & 1680 & $1.1{\times}10^{-20}$ & *** \\
\bottomrule
\end{tabular}
\end{table}

Two patterns emerge.
\textbf{(i) Gains over Base are significant in all nine model--task pairs} ($p_{\mathrm{Holm}} \leq 2.2{\times}10^{-3}$, paired $t$-test), with the largest adjusted $p$-values on Qwen3-8B/GPQA-D and Llama-3.1-8B/GPQA-D (both $p_{\mathrm{Holm}} = 2.2{\times}10^{-3}$) and the smallest on Llama-3.2-3B/MMLU-Pro ($p_{\mathrm{Holm}} = 4.1{\times}10^{-108}$).
The Wilcoxon test agrees with the $t$-test in every cell.
\textbf{(ii) Advantages over BoN are significant where they are largest.}
Five of the six Llama cells are significant after correction ($p_{\mathrm{Holm}} \leq 2.5{\times}10^{-2}$), the exception being Llama-3.1-8B/GPQA-D ($+4.4$, $p_{\mathrm{Holm}} = 1.3{\times}10^{-1}$), while the three Qwen cells---two near-ties and a $+2.2$ gap---do not reach significance, consistent with the endpoint analysis in Section~\ref{sec:compute}.


\subsection{Hyperparameter Analysis}
\label{app:hyperparam}


The M-step steering has two hyperparameters: the total strength $\mu$ and the confidence threshold $\alpha$ that masks low-confidence tokens (Section~\ref{sec:mstep}).
Table~\ref{tab:hyperparam} sweeps them on Qwen3-8B over GPQA-Diamond (greedy, all other settings identical to the main experiments).

\begin{table}[h]
\centering
\caption{\textbf{Hyperparameter sweep on Qwen3-8B / GPQA-Diamond} (accuracy \%, greedy; base $=49.0$). Panel (a) sweeps the total strength $\mu$ at $\alpha{=}0.1$; panel (b) sweeps the confidence threshold $\alpha$ at $\mu{=}2.4$. The main-experiment default ($\mu{=}2.4$, $\alpha{=}0.1$) attains $56.6$ ($+7.6$).}
\label{tab:hyperparam}
\small
\begin{tabular}{c cc}
\toprule
\multicolumn{3}{l}{\textit{(a) Total strength $\mu$ ($\alpha{=}0.1$)}} \\
\midrule
$\mu$ & EMO & $\Delta$ vs.\ base \\
\midrule
1.6 & 56.6 & +7.6 \\
2.0 & 54.5 & +5.6 \\
2.4 & 56.6 & +7.6 \\
2.8 & 56.1 & +7.1 \\
3.2 & 54.5 & +5.6 \\
\midrule
\multicolumn{3}{l}{\textit{(b) Confidence threshold $\alpha$ ($\mu{=}2.4$)}} \\
\midrule
$\alpha$ & EMO & $\Delta$ vs.\ base \\
\midrule
0.05 & 50.0 & +1.0 \\
0.1  & 56.6 & +7.6 \\
0.2  & \B{58.6} & \B{+9.6} \\
0.4  & 57.1 & +8.1 \\
\bottomrule
\end{tabular}
\end{table}

\paragraph{Strength is insensitive.}
Across a $2\times$ range of the total strength ($\mu \in [1.6, 3.2]$) the gain stays within $5.6$--$7.6$ points with no monotone trend: $\mu=1.6$ and the default $\mu=2.4$ tie for the highest observed accuracy ($56.6$), and no setting collapses.
This indicates that the KL-gradient signal (Section~\ref{sec:mstep}) is a robust direction---once the update points toward the optimized context, its magnitude is of secondary importance.

\paragraph{The confidence constraint is what matters.}
The picture changes qualitatively for $\alpha$.
At $\alpha{=}0.05$ the gain collapses to $+1.0$: with too permissive a mask, the steering vector acts on tokens where the optimized context expresses almost no preference, injecting noise rather than signal.
At $\alpha{=}0.1$--$0.4$ the gain is stable ($+7.6$ to $+9.6$, with the best value at $\alpha{=}0.2$), i.e., the constraint only needs to separate confident tokens from noise, not to be tuned precisely.
We conclude that the confidence mask, not the step size, is the load-bearing hyperparameter of the M-step, and keep the conservative default $\alpha{=}0.1$ throughout the main experiments (where it already attains the $+7.6$ gain reported above).

\subsection{Prompt Templates}
\label{app:prompts}


This section documents the four prompts that drive the E-step (Section~\ref{sec:method}): the judge that scores responses, the backward engine that produces the textual gradient, the optimizer that rewrites the context, and the answer-leakage filter that guards the whole loop.
All prompts are task-agnostic (the same templates are used for all six benchmarks); the question stem and option block are the only task-specific inputs.

\subsubsection{Judge Scoring Prompt}
\label{app:prompt_judge}

The judge prompt (\texttt{\_EVAL\_SYSTEM\_PROMPT} in \texttt{trainer.py}) instructs the judge model to grade a response on a 0--10 scale while remaining blind to the gold answer.
Its key requirements are: (i) solve the problem \emph{independently} before grading, never assuming the model's answer is correct; (ii) 10 means the final answer is correct, 0 means wrong or missing, with partial credit (1--9) for substantive progress on incorrect solutions; (iii) the feedback must locate the methodological error and how to fix it \emph{without ever revealing or hinting at the final answer} (no numbers, letters, or expressions); (iv) no restating of the question or response.

\begin{quote}
\small
\texttt{Question: \{question\}\textbackslash n\textbackslash nResponse: \{response\}\textbackslash n\textbackslash nEvaluate this response. Provide detailed feedback and a score.}
\end{quote}

The judge returns its output in the format \texttt{<FEEDBACK>...</FEEDBACK><SCORE>0-10</SCORE>}; the score drives judge-select (Appendix~\ref{app:settings}) and the feedback becomes the input to the backward pass below.

\subsubsection{Backward (Text Gradient) Prompt}
\label{app:prompt_backward}

The backward prompt (\texttt{BACKWARD\_SYSTEM\_PROMPT} in \texttt{llm\_backward\_prompts.py}) asks the engine to act as the ``gradient'' with respect to the context variable: given the current system prompt, the question, the model's response, and the judge feedback, produce \emph{criticism of the context}---strategies, explanations, and methods to change---but never a rewritten version of the variable (that is the optimizer's job).
It carries the same no-leakage and conciseness clauses as the judge prompt.
The gradient is injected through a structured conversation template:

\begin{quote}
\small
\texttt{<LM\_SYSTEM\_PROMPT>\{system\_prompt\}</LM\_SYSTEM\_PROMPT>} \\
\texttt{<LM\_INPUT>\{prompt\}</LM\_INPUT>} \\
\texttt{<LM\_OUTPUT>\{response\}</LM\_OUTPUT>} \\
\texttt{<FEEDBACK>\{feedback\}</FEEDBACK>}
\end{quote}

Note that the full response is embedded in the backward prompt; this is the main driver of the judge-context length budget (${\sim}14.5$k tokens at peak for our longest responses).

\subsubsection{Optimizer (Context Rewrite) Prompt}
\label{app:prompt_optimizer}

The optimizer prompt (\texttt{OPTIMIZER\_SYSTEM\_PROMPT} in \texttt{optimizer\_prompts.py}) implements the TGD update: improve the context variable using the feedback, returning the improved variable wrapped in \texttt{\{new\_variable\_start\_tag\} \{improved variable\} \{new\_variable\_end\_tag\}} tags so it can be extracted and replace the old context directly, with an instruction to keep the improved variable concise.
The update instruction passes the variable's role, its current value, and the gradient from the backward pass.
What shapes the rewrite most is the \emph{role description} attached to the context variable:

\begin{quote}
\small
\emph{``system prompt to guide the model's behavior. A good system prompt establishes the model as a domain expert relevant to the task (e.g.\ `You are an expert chemist \ldots'), states the key reasoning approach and pitfalls to avoid, and specifies the required output format. Avoid generic phrases; be specific to the question's domain. Do NOT reveal or hint at the actual answer.''}
\end{quote}

This steers the E-step away from generic encouragements (``think carefully'') toward substantive, domain-specific guidance, which is what the M-step steering vector can then convert into a decoding-time correction.
Optional momentum (past iterations), constraints, and in-context-example sections exist in the template but are disabled in all experiments reported in this paper.

\subsubsection{Answer-Leakage Filter}
\label{app:prompt_leak}

Two complementary guards keep the optimization signal free of the answer.
First, the judge never receives the gold answer, and its prompt explicitly forbids the feedback from revealing or hinting at the final answer (Appendix~\ref{app:prompt_judge}).
Second, every context produced by the optimizer is passed through a regex filter (\texttt{\_ANSWER\_LEAK\_PATTERNS} in \texttt{optimizer.py}) that rejects rewrites containing concrete answers---e.g.\ ``the answer is X'' patterns, contrastive ``the answer is X, not Y'' constructions, or specific numbers/letters---in which case the previous context is retained.
The v2 filter no longer false-positives on the output-format directive \texttt{Answer: \$LETTER\$}: a literal \texttt{\$} is only treated as numeric leakage when followed by a digit.
Together these guards ensure that neither the judge feedback nor the optimized context can smuggle the gold answer into the generation path, so any accuracy gain is attributable to the optimization itself rather than to answer leakage.

\end{document}